\newcommand{\edit}[1]{\textcolor{black}{#1}}

\newcommand{\zbedit}[1]{\textcolor{black}{#1}}
\newenvironment{add}
{\color{black}}
{}

\renewcommand{\footnote}[1]{}
\newif\ifcycomment
\cycommentfalse 

\newcommand{\cyfoot}[1]{\ifcycomment\footnote{#1}\fi}

\documentclass[sigconf]{acmart}

\usepackage{amsthm}
\usepackage{multirow}
\newtheorem{assumption}{Assumption}
\usepackage{algorithm}
\usepackage{algpseudocode}
\algrenewcommand\algorithmicrequire{\textbf{Input:}}
\algrenewcommand\algorithmicensure{\textbf{Output:}}

\AtBeginDocument{%
  }

\setcopyright{cc}
\setcctype{by}
\copyrightyear{2026}
\acmYear{2026}
\acmDOI{10.1145/3770855.3818042}
\acmConference[KDD 2026] {Proceedings of the 32nd ACM SIGKDD Conference on Knowledge Discovery and Data Mining V.2}{August 9--13, 2026}{Jeju Island, Republic of Korea.}
\acmBooktitle{Proceedings of the 32nd ACM SIGKDD Conference on Knowledge Discovery and Data Mining V.2 (KDD 2026), August 9--13, 2026, Jeju Island, Republic of Korea}
\acmISBN{979-8-4007-2259-2/2026/08}

\begin{document}

\title{SIMS: Scale-Invariant Merit-Function-Based Scalarization for Multi-Task Learning}

\author{Zebin Chen}
\authornote{Both authors contributed equally to this research.}
\email{12432660@mail.sustech.edu.cn}
\affiliation{%
  \institution{Southern University of Science and Technology}
  \city{Shenzhen}
  \country{China}
}

\author{Fei Xing}
\authornotemark[1]
\email{fxing8-c@my.cityu.edu.hk}
\affiliation{%
  \institution{City University of Hong Kong}
  \city{Hong Kong SAR}
  \country{China}}

\author{Yang Chen}
\email{cheny2023@mail.sustech.edu.cn}
\affiliation{%
  \institution{Southern University of Science and Technology}
  \city{Shenzhen}
  \country{China}
}

\author{Hua Liu}
\email{liuh5@sustech.edu.cn}
\affiliation{%
 \institution{Southern University of Science and Technology}
  \city{Shenzhen}
  \country{China}
}

\author{Andy H.F. Chow}
\email{andychow@cityu.edu.hk}
\affiliation{%
  \institution{City University of Hong Kong}
  \city{Hong Kong SAR}
  \country{China}}

\author{Yuhua Qian}
\email{jinchengqyh@126.com}
\affiliation{%
  \institution{Shanxi University}
  \city{Taiyuan}
  \country{China}}

\author{Yu Zhang}
\email{yu.zhang.ust@gmail.com}
\authornote{Yu Zhang is the corresponding author.}
\affiliation{%
  \institution{Southern University of Science and Technology}
  \city{Shenzhen}
  \country{China}}

\renewcommand{\shortauthors}{Zebin Chen, Fei Xing, Yang Chen, Hua Liu, Andy H.F. Chow, Yuhua Qian, Yu Zhang.}

\begin{abstract}
Multi-task learning (MTL) requires navigating unavoidable trade-offs among competing objectives.
This paradigm is frequently formulated as multi-objective optimization (MOO), where the scalarization is favored to reduce an MOO problem to a single objective.
We empirically find that existing merit-function-based scalarization approaches are sensitive to the relative scales of different objectives in practical MTL, where task losses commonly differ by orders of magnitude. 
The optimization process often favors objectives with larger scales even though the underlying Pareto  optimal solutions remains invariant to rescaling (i.e., multiplying an objective by a positive constant).
To address this issue, we propose \textbf{S}cale-\textbf{I}nvariant \textbf{M}erit-function-based \textbf{S}calarization (\textbf{SIMS}) for MTL. Specifically, SIMS adopts a transformation-induced merit function to convert the MOO problem of MTL to a single objective \edit{that renders optimization invariant to} the magnitudes of losses.
Theoretically, we prove that the requirement for scale invariance uniquely determines this transformation to be logarithmic. We further show that \edit{this} general transformation-induced merit function preserves weak Pareto optimality and admits a smooth surrogate with controllable approximation error.
Extensive experiments on representative multi-task benchmarks demonstrate that SIMS consistently outperforms existing scalarization methods and achieves state-of-the-art performance.
\end{abstract}

\ccsdesc[500]{Computing methodologies~Multi-task learning}

\keywords{Multi-Task Learning, Multi-Objective Optimization, Merit Function}


\maketitle
\newcommand\kddavailabilityurl{https://doi.org/10.5281/zenodo.20377485}
\ifdefempty{\kddavailabilityurl}{}{
\begingroup\small\noindent\raggedright\textbf{Resource Availability:}\\
The source code of this paper has been made publicly available at
\url{\kddavailabilityurl}. The corresponding GitHub repository is available at
\url{https://github.com/Chen-zb/SIMS}.
\endgroup
}
\section{Introduction}

Multi-task learning (MTL) aims to develop a model capable of learning from multiple related tasks simultaneously~\cite{caruana1997multitask,zhang2021survey}. By leveraging shared representations across tasks, MTL enhances \edit{data efficiency and generalization~\cite{caruana1997multitask,ruder2017overview}}
, and has established itself as a standard paradigm in applications such as dense prediction\edit{~\cite{vandenhende2021multi}}
, \edit{autonomous driving~\cite{chowdhuri2019multinet,wang2025review} and recommendation systems~\cite{zhao2019recommending}}
. However, tasks in MTL are rarely perfectly aligned. The optimization of one task may degrade the performance of another because the learning process often requires navigating unavoidable trade-offs among competing objectives.

A significant body of research addresses this challenge through architectural design~\cite{liu2019end,misra2016cross}, loss weighting~\cite{kendall2018multi,vandenhende2021multi}, or gradient balancing~\cite{chen2018gradnorm,guo2018dynamic}\cyfoot{***cy: Correlate the references with architectural design, loss weighting, or gradient balance. \edit{Done}}. While these methods are effective in many applications, they are typically heuristic, relying on optimizing a proxy objective that \edit{lacks theoretical guarantees for converging to the optimal solution}\footnote{***at current position, MOO is not mentioned. Suddenly mentioning `Pareto optimal solutions' seems a bit strange. \edit{Put it in another way}} when tasks are competing.
To address these limitations, recent efforts have sought to 
cast the objective function of MTL as a multi-objective optimization (MOO) problem, where each task loss serves as an individual objective with the goal of reaching \edit{Pareto optimal solutions~\cite{sener2018multi,marler2004survey}}\cyfoot{***cy: check `compromise'. Maybe `finding the Pareto optimal solutions'? \edit{have changed to "the Pareto optimal solutions"}}\cyfoot{***cy: add ref. \edit{Done}}. MOO-based MTL methods are generally categorized into the \edit{adaptive gradient} approach~\cite{sener2018multi} and scalarization approach~\cite{lin2024smooth}.  \edit{Unlike adaptive gradient approach, scalarization approach}\cyfoot{***cy: Scalarization-based approaches? \edit{Done, all change to scalarization approaches}} is particularly appealing because it transforms the MOO problem into standard single-objective training by optimizing an aggregated objective, \edit{thereby avoiding the high computational and memory overheads associated with manipulating per-task gradients.}\footnote{***this reason cannot demonstrate that the scalarization approach is appealing. \edit{The advantages of the scalarization approaches over the adaptive gradient approach have been added}}

In this paper, we focus on the merit-function-based\cyfoot{***cy: `merit-function based' or `merit-based'. Be consistent. \edit{merit-function-based}} scalarization approach, which eliminates the need for manually tuning task weighting while maintaining a principled connection to the Pareto optimality. However, in practice, the scales of task losses in various tasks often differ by orders of magnitude. Based on a motivating example shown in Section \ref{sec:motivation}, we observe that existing merit-function-based scalarization methods are often sensitive to the scales of task losses.
Although the change of scales of objectives via the rescaling \edit{(i.e., multiplying each objective by a positive constant, akin to a change of units)} does not change the Pareto optimal solutions, 
\edit{it can substantially affect the optimization of scalarization methods. Specifically, the rescaling changes the relative magnitudes of objectives and also gradients, which can alter the optimization trajectory and may lead to unsatisfactory solutions,  
which strongly favor only a few objectives while sacrificing others, with worse average performance across tasks. 
Moreover, we find that the change of the scales of objective functions may even make the optimization procedure divergent. }\cyfoot{***cy: what is positive rescaling? Give a definition. \edit{Done}}\cyfoot{***cy: More explanations of the phenomena in the toy example. \edit{have added more details}}\cyfoot{***cy: why bias the obtained Pareto point is bad. Explain it. \edit{I conclude that it will lead to extreme trade-off solutions}}

To alleviate the above issues in scalarization methods and achieve the scale-invariance in objective functions, we propose the \textbf{S}cale-\textbf{I}nvariant \textbf{M}erit-function-based \textbf{S}calarization (SIMS) method, which \edit{preserves a principled characterization of weak Pareto optimality, }\cyfoot{***cy: check `Pareto-relevant'. \edit{change to "preserves a principled characterization of weak Pareto optimality"}}\footnote{***our method cannot find all the Pareto-optimal solutions. So here we cannot say `the full Pareto set'. Moreover, `Pareto set' means `Pareto front'? \edit{have changed to "preserves a principled characterization of weak Pareto optimality"}} while remaining intrinsically \edit{invariant to the scales of task losses}.\cyfoot{***cy: Change all `robust' to another word in this paper. \edit{change to "invariant to task-wise loss scales"}} 
Specifically, in SIMS, we \edit{propose} a \edit{transformation}-induced\cyfoot{***cy: `Deformation' is mostly used to refer to physical deformation. \edit{change to transformation}} merit function, \edit{which} applies a monotonically increasing transformation function to each task loss. Theoretically, we prove that, to achieve the scale-invariance, the transformation function in SIMS should be logarithmic functions. We further provide theoretical analysis to show that the transformation-induced merit function in SIMS maintains a valid characterization of weak Pareto optimality and admits a smooth surrogate with controllable approximation error.
Empirically, extensive experiments demonstrate that the proposed SIMS method consistently outperforms existing scalarization approaches and achieves state-of-the-art performance.

The main contributions of this work are summarized as follows.
\begin{itemize}
\item We empirically demonstrate that existing \edit{merit-function-based }scalarization methods are often sensitive to the scales of task losses and show that scale-invariance\cyfoot{***cy: Is `scale invariance' a general term? Otherwise, you need to define it. \edit{will define it on Motivating examples}} is essential for obtaining consistent solutions.\footnote{***we are the first to make such observation for scalarization methods? If not, then do not list as a contribution. \edit{Added the attributive of the scalarization methods}}
\item We propose the SIMS method, a scalarization approach which can \edit{characterize weak Pareto optimality while ensuring the scale-invariant property}.
\item We theoretically analyze the proposed SIMS method, including the formulation of the transformation function, the weak Pareto optimality, and the approximation analysis.
\item Extensive experiments on multi-task benchmarks demonstrate the effectiveness of the proposed SIMS method.
\end{itemize}

\section{Related work}
\noindent \textbf{MOO for MTL.} 
To formulate the objective function of MTL as an MOO problem, each task loss corresponds to an objective. To better handle task conflicts, the adaptive gradient approach has been widely studied to construct a descent direction that balances multiple gradients of tasks. For example, MGDA \cite{sener2018multi} finds a common descent direction to satisfy the Pareto stationarity \citep{desideri2012multiple}.\footnote{***add references \edit{Done}} Subsequent methods such as CAGrad \cite{liu2021conflict}, PCGrad \cite{yu2020gradient}, Nash-MTL \cite{navon2022multi}, and FairGrad~\cite{ban2024fair} improve the stability or fairness via respective gradient aggregation operations. Although these adaptive gradient methods provide strong theoretical guarantees, they typically require accessing or manipulating per-task gradients and storing them during optimization, which leads to high computational and memory overhead. 
Due to these issues, we focus on the scalarization approach that combines multiple task losses to a single objective.

\vspace{0.5\baselineskip}
\noindent \textbf{Scalarization approach.}
In MOO, the scalarization approach converts multiple objectives to a single scalar objective to be optimized~\cite{marler2004survey}. 
A substantial body of prior work in MTL is built upon the linear scalarization, where different task losses are linearly combined to form the objective function. 
Representative methods under this paradigm include Equal Weight (EW) with an identical task weight, Uncertainty Weighting (UW)~\cite{kendall2018multi}, Dynamic Weight Averaging (DWA)~\cite{liu2019end}, and GradNorm~\cite{chen2018gradnorm}. 
However, the linear scalarization could miss all the solutions on the non-convex part of the Pareto front~\cite{das1997closer}. 
To alleviate this issue, the Smooth Tchebycheff (STCH) method~\cite{lin2024smooth} introduces a smooth formulation to resolve the non-differentiability of the classical weighted Tchebycheff method, which enables efficient gradient-based optimization, while preserving the theoretical capability to capture non-convex Pareto fronts. Moreover, the FOOPS method~\cite{chen2025efficient} leverages a merit function for MTL to provide an explicit scalar measure of the Pareto optimality. 
\zbedit{Some recent methods have also considered scale mismatch in MTL. For example, GLS \cite{chennupati2019multinet++} reduces the influence of loss magnitudes through geometric loss, but it does not explicitly characterize whether the current solution admits further joint improvement or reaches Pareto optimality. Other methods, such as Nash-MTL \cite{navon2022multi} and DB-MTL \cite{lin2025dual} address related issues from the gradient-level perspective.}
Unlike these methods, the proposed SIMS method aims to solve the scale-sensitive issue of existing merit functions.

\section{Preliminary}
\subsection{MTL as MOO}
An MTL problem involves a set of $m$ related tasks and can be formulated as an MOO problem as
\begin{equation}
    {\min}_{\theta\in \mathbb{R}^d} \; \boldsymbol{\mathcal{L}}(\theta)
    = \bigl( \mathcal{L}_1(\theta), \mathcal{L}_2(\theta), \ldots, \mathcal{L}_m(\theta) \bigr)^\top,
    \label{eq:1}
\end{equation}
where $\mathcal{L}_i(\theta):\mathbb{R}^d\to\mathbb{R}_+$ denotes the task loss of the $i$-th task and $\theta$ denotes the shared model parameters across tasks. 
For problem~\eqref{eq:1}, there is usually no single solution that can simultaneously minimize all objectives. Instead, we consider the Pareto optimality for problem~\eqref{eq:1}, which is defined as follows.
\begin{definition}[Pareto optimality]\label{def:pareto_opt}
A solution $\theta^\star \in \mathbb{R}^d$ is \emph{Pareto optimal} to problem~\eqref{eq:1}, if there does not exist any other solution $\theta \neq \theta^\star$ such that $\mathcal{L}_i(\theta) \le \mathcal{L}_i(\theta^\star)$ for all $i \in [m]$ and $\mathcal{L}_j(\theta) < \mathcal{L}_j(\theta^\star)$ for at least one index $j$.
\end{definition}

\edit{The set of all Pareto optimal solutions is called the Pareto set. The image of the Pareto set in the loss space is called the Pareto front.} We also consider a weaker notion of Pareto optimality, namely weak Pareto optimality.\footnote{***briefly introduce the Pareto front, which is already used in this paper.\edit{Done}}

\begin{definition}[Weak Pareto optimality]\label{def:weak_pareto_opt}
A solution $\theta^\star \in \mathbb{R}^d$ is \emph{weakly Pareto optimal} to problem~\eqref{eq:1}, if there does not exist any other solution $\theta \in \mathbb{R}^d$ such that $\mathcal{L}_i(\theta) < \mathcal{L}_i(\theta^\star)$ for all $i \in [m]$.
\end{definition}




\begin{figure*}[!hpbt]
  \centering
  \includegraphics[width=0.97\linewidth]{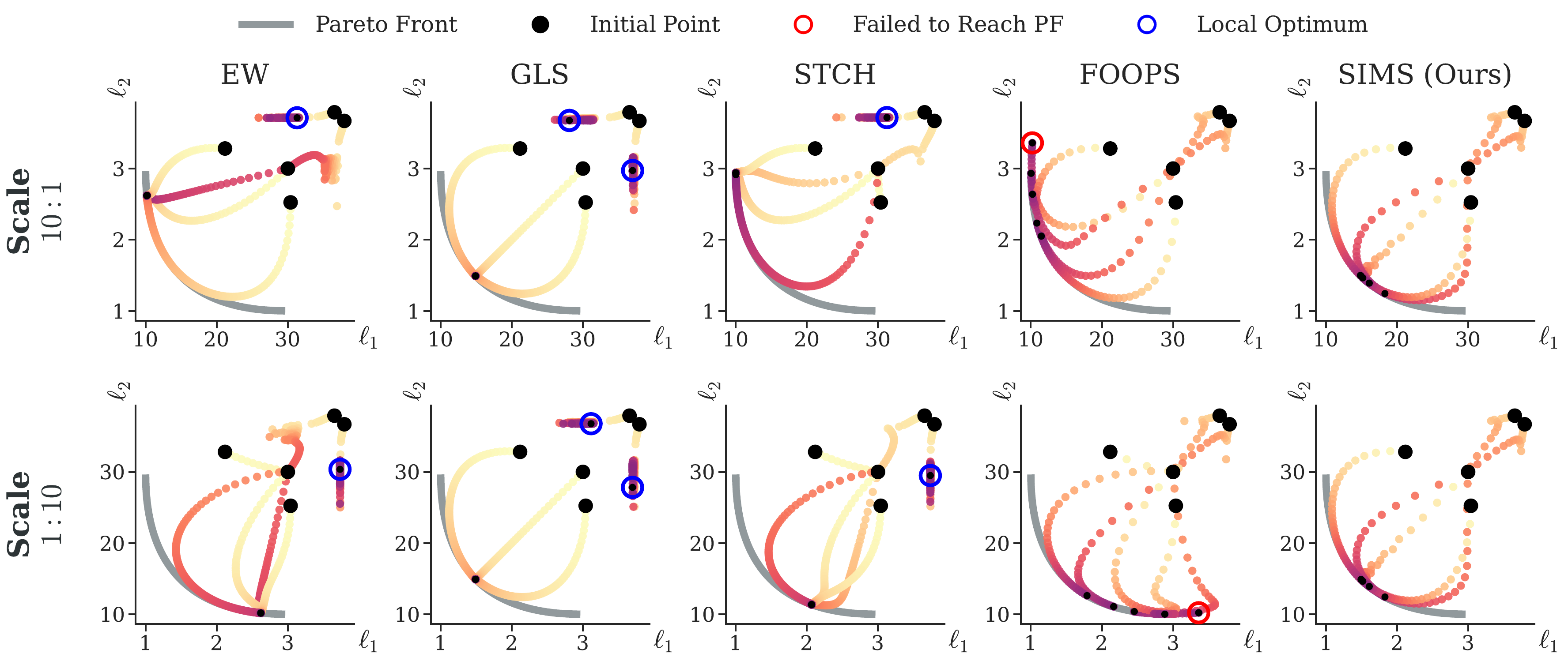}
  \caption{The visualization of optimization trajectories of different methods in the loss space. Gray: Pareto front. Black: initializations. Colored curves: training trajectories. Under both objective-scale settings (Scale 10:1 and 1:10), the Pareto solution set is unchanged. EW, STCH, and FOOPS show clear sensitivity to objective scaling. \zbedit{GLS may get trapped in local optima and converge away from the Pareto front}. In contrast, SIMS (Ours) behaves consistently across scales, producing similar trajectories and converging to comparable regions.
  }
  \label{fig:toy-scale}
\end{figure*}

\subsection{Merit Functions}
Merit functions~\cite{patriksson1997merit,charitha2013note,pappalardo2016merit} provide a surrogate function that connects MOO with single-objective optimization. Merit functions are defined as follows.
\begin{definition}[Merit Function]
\label{def_merit_function}
A function $p:\mathbb{R}^d\to\mathbb{R}$ is a merit function for problem~\eqref{eq:1}
if it satisfies the following properties:
\begin{itemize}
    \item $p(\theta)\ge 0$ for all $\theta\in\mathbb{R}^d$;
    \item $\theta^\star$ is a weakly Pareto optimal solution of problem~\eqref{eq:1} if and only if $p(\theta^\star)=0$.
\end{itemize}
\end{definition}


 


A simple example of the merit function is the maximum-loss merit function~\cite{miettinen1999nonlinear}, which is defined as
\[
p_{\max}(\theta)
=
\max_{i\in[m]} \big( \mathcal{L}_i(\theta) - \mathcal{L}_i^{\inf} \big),
\]
where $[m]$ denotes the set of integers from 1 to $m$, and $\mathcal{L}_i^{\inf}= \inf_{\theta\in\mathbb{R}^d} \mathcal{L}_i(\theta)$ denote the minimal attainable value of the $i$-th objective. 
We can prove that the maximum-loss merit function satisfies Definition \ref{def_merit_function}. 
Moreover, a weighted extension of the maximum-loss merit function leads to the Chebyshev merit function, which is used in MOO~\cite{ehrgott2005multicriteria,lin2024smooth}. Specifically, the Chebyshev merit function\footnote{***Chebyshev-type merit function or Chebyshev merit function? to be consistent. \edit{Done}} is defined as
\[
p_{\mathrm{CH}}(\theta)
=
\max_{i\in[m]} w_i \big( \mathcal{L}_i(\theta) - \mathcal{L}_i^{\inf} \big).
\]
Compared with $p_{\max}(\theta)$, $p_{\mathrm{CH}}(\theta)$ preserves the properties of the merit function, while allowing different objectives to be emphasized through the weights.


To analyze the merit function, we need the stationary property, which is defined as follows.

\begin{definition}
A solution $\theta$ is an $\epsilon$-stationary point ($\epsilon \ge 0$) of a differentiable function $p$ if its gradient $\nabla p(\theta)$ satisfies $\|\nabla p(\theta)\| \le \epsilon$, where $\|\cdot\|$ is the Euclidean norm. If $\epsilon = 0$, then $\theta$ is a stationary point.
\end{definition}

\section{Motivating Examples}
\label{sec:motivation}

In MTL, the task losses across tasks are often in different scales. This phenomenon can arise from heterogeneous loss functions and distinct types of tasks (e.g., regression versus classification tasks)~\cite{kendall2018multi}, and is further enlarged by implementation choices such as the reduction convention \edit{(e.g., sum versus mean aggregation)}\footnote{***what does reduction convention mean? \edit{Added explanations}} and task-specific normalization~\cite{ren2015faster,he2017mask}. 
Although such scale differences across tasks are common, they \edit{do not theoretically affect}\footnote{***revise this phrase \edit{Done}} the Pareto optimality because any rescaling of objectives will not change the Pareto optimal solutions. 
In contrast, scalarization methods are often sensitive to the scales of task losses because rescaling changes the gradient magnitudes across different tasks, which can alter the optimization trajectory to a different solution.  
When task losses differ by orders of magnitude, \edit{existing merit-function-based }scalarization methods\footnote{***you need to narrow down to merit-function-based scalarization methods. \edit{Done}} can be dominated by objectives with larger scales, pushing the optimization toward Pareto-optimal yet extreme trade-off solutions, where the model favors a few tasks while substantially sacrificing the remaining ones. This often yields unsatisfactory performance for MTL.

To see that, we study a two-task synthetic problem~\cite{navon2022multi}, where the Pareto front is available in closed-form solutions as indicated by the gray curve in Fig.~\ref{fig:toy-scale}. 
We construct this toy example by introducing two task weights of the objectives.
Specifically, we multiply the two task losses by different scaling factors, i.e., $10:1$ and $1:10$. This setup emulates a mismatch in measurement units across task losses, allowing us to examine how various methods behave under different task scales. 
Note that this rescaling does not change the Pareto front. 
We run from multiple initializations (in black dots) and compare SIMS with linear scalarization method with equal-weight (EW), geometric loss strategy (GLS) \cite{chennupati2019multinet++} and two merit-function-based scalarization methods (i.e., STCH \cite{lin2024smooth} and FOOPS \cite{chen2025efficient}). More experimental details are provided in Appendix ~\ref{appdix:example}.

Fig.~\ref{fig:toy-scale} shows that although the Pareto front remains unchanged, the EW, STCH, and FOOPS methods exhibit significant sensitivity to the scales of task losses. Specifically, under these two scaling factors (10:1 and 1:10), EW, STCH, and FOOPS exhibit different optimization trajectories, respectively, and they consistently drift toward the task loss with the larger scale. 
We further observe that the FOOPS method exhibits optimization instability and even fails to converge to a stationary solution in some cases (i.e., the point circled in red). 
\zbedit{Although GLS is invariant to positive task-wise loss rescaling due to its geometric-loss formulation, it still suffers from local optima in this case and converges to solutions away from the Pareto front under both scaling settings. This indicates that scale insensitivity alone is insufficient without a merit-function-based criterion that explicitly characterizes the possibility of further joint improvement.}
In contrast, 
the proposed SIMS method exhibits a scale-invariant property due to the proposed transformation-induced merit function, 
leading to a more consistent convergence behavior under different scalings. \edit{This aligns with our goal of obtaining solutions that remain unchanged under different loss scales.} 
\edit{Overall, to mitigate extreme trade-off solutions and achieve better average performance across tasks, scalarization methods in MTL should exhibit a certain degree of scale invariance.}

\section{Methodology}

In this section, we present the proposed SIMS method.

\subsection{Formulation}
\label{section:formulation}
To achieve the scale invariance where the induced optimization dynamics is less affected by arbitrary task-wise scales, we propose a transformation-induced merit function. 
Specifically, we first apply a mapping $\psi(\cdot)$ to the original task loss $\mathcal{L}_i(\theta)$, and 
define the transformed task loss as
\begin{align}
\tilde{\mathcal{L}}_i(\theta)=\psi\!\left(\mathcal{L}_i(\theta)\right)\notag,
\end{align}
where $\psi(\cdot)$ is a transformation function shared by different tasks. 
Moreover, $\psi(\cdot)$ should be a monotonically increasing function so that 
minimizing $\tilde{\mathcal{L}}_i(\theta)$ is equivalent to minimizing $\mathcal{L}_i(\theta)$.


Then, based on 
the MOO formulation in problem~\eqref{eq:1}, 
the new MOO problem based on $\{\tilde{\mathcal{L}}_i(\theta)\}$ is formulated as
\begin{equation}
\min_{\theta \in \mathbb{R}^d} \! \tilde{\boldsymbol{\mathcal{L}}}(\theta)\!
    =\!
    \bigl(
        \psi(\mathcal{L}_1(\theta)),\!
        \psi(\mathcal{L}_2(\theta)),\!
        \ldots,
        \psi(\mathcal{L}_m(\theta))
    \bigr)^\top\!.
    \label{eq:2}
\end{equation}
According to the definition of the merit function in Definition \ref{def_merit_function}, a merit function associated with problem~\eqref{eq:2} should be non-negative and attain zero if and only if the solution is (weakly) Pareto optimal. Under this requirement and assuming lower semicontinuity of $\{\mathcal{L}_i\}$, the proposed transformation-induced merit function is formulated as
\begin{equation}
    \bar{u}(\theta)=\sup_{\theta' \in \mathbb{R}^d}\min_{i \in [m]}
    \bigl\{\psi(\mathcal{L}_i(\theta))-\psi(\mathcal{L}_i(\theta'))\bigr\}.
    \label{eq:ubar}
\end{equation}
$\bar{u}(\theta)$ measures the largest improvement that some other feasible solution can simultaneously achieve across all objectives over $\theta$. 

\subsection{Properties}
\label{sec5}

In the following theorem, we prove that $\bar{u}$ is a valid merit function in the sense of weak Pareto optimality. 
\begin{theorem}
\label{theo:u0_merit}
Suppose that $\psi(\cdot)$ is a monotonically increasing function. Then we have $\bar{u}(\theta)\ge 0$ for all $\theta \in \mathbb{R}^d$. Moreover, $\theta \in \mathbb{R}^d$ is weakly Pareto optimal for problem~(1) if and only if $\bar{u}(\theta)= 0$. Hence, $\bar{u}$ is a merit function.
\end{theorem}

As shown in the motivating example, a scalarization method for MTL is expected to be invariant to the scale of task losses. We formalize the definition of this scale-invariance requirement as follows.

\begin{definition}[Scale-Invariance]
\label{def:scale_invariance}
Consider an MTL problem with task losses $\{\mathcal{L}_i\}_{i=1}^m$. For any positive constants $c_i > 0$, the rescaled losses is defined as $\bar{\mathcal{L}}_i(\theta) = c_i\,\mathcal{L}_i(\theta)$. Then, a scalarization MTL method is said to be \emph{invariant to the scaling transformation} if the scalarized objective induced by $\{ \bar{\mathcal{L}}_i(\theta)\}_i^m$ differs from the original scalarized objective only by an additive constant that is independent of \(\theta\).
\end{definition}

Definition~\ref{def:scale_invariance} requires scalarization MTL methods to be invariant to task-wise rescaling of losses, so that changing the units or magnitudes of individual task losses does not affect the preference of the results. Without such a property, a task may dominate the optimization merely due to the large scale of its task loss. 

For the functional form of the transformation function $\psi(\cdot)$, in the following theorem, we prove that the scale invariance property of the transformation-induced merit function $\bar u(\theta)$ makes $\psi(\cdot)$ a logarithmic function.

\begin{theorem}
\label{thm:log_characterization}
When solving an MTL problem by minimizing $\bar{u}(\theta)$ in Eq.~\eqref{eq:ubar}, if the optimal solution is required to be invariant to the scaling transformation, then $\psi$ must be of the form $\psi(x) = a \ln x + b$ for some constants $a > 0$ and $b \in \mathbb{R}$.
\end{theorem}

The constants $a$ and $b$ in $\psi(\cdot)$ only introduce a positive affine rescaling and therefore do not change the ordering of objective values. Consequently, they do not affect the solution set of problem~\eqref{eq:ubar}. Hence, there is no need to tune $a$ and $b$ in practice. Without loss of generality, we adopt a simple choice in experiments, that is, $\psi(x) = \ln x$, which satisfies the required monotonicity and scale-invariance properties.

\subsection{Approximation}
\label{sec:app}
Since $\bar{u}(\theta)$ has been proved in Theorem \ref{theo:u0_merit} to be a merit function of problem \eqref{eq:2}, we can minimize $\bar{u}(\theta)$ to obtain a weakly Pareto optimal solution. 
However, due to the pointwise $\min$ operator, $\bar{u}$ is generally nonsmooth and non-differentiable~\cite{boyd2004convex}. 
This prevents the direct use of gradient descent and instead necessitates subgradient-based methods, which typically converge more slowly~\cite{goffin1977convergence}. 

To this end, we construct a smoothed and regularized surrogate of $\bar{u}(\theta)$ inspired by classical smoothing techniques (i.e.,log-sum-exp smoothing~\cite{nesterov2005smooth,beck2017first}) for nonsmooth optimization. 
Note that the proposed optimization procedure is applicable to any transformation function $\psi$. 
Specifically, we define
\begin{align}
h_{\lambda, \tau}(\theta,\theta')
=&\tau \ln \left(\sum_{i=1}^{m}\exp\left(\frac{\psi(\mathcal{L}_i(\theta'))-\psi(\mathcal{L}_i(\theta))}{\tau}\right)\right)\notag+\frac{\lambda}{2}\|\theta-\theta'\|^2 \nonumber\\
v_{\lambda, \tau}(\theta)=&-\min_{\theta'\in \mathbb{R}^d}h_{\lambda,\tau}(\theta,\theta'),\label{eq:v}
\end{align}
where $\|\cdot\|$ denotes the Euclidean norm, $\tau>0$ is a smoothing parameter, and $\lambda\ge 0$ is a regularization parameter. Therefore, problem~\eqref{eq:1} is converted to minimizing $v_{\lambda,\tau}$, which is equivalent to solving the following saddle-point problem as
\begin{equation}
    \max_{\theta\in \mathbb{R}^d}\min_{\theta'\in \mathbb{R}^d}
    h_{\lambda,\tau}(\theta,\theta').
    \label{eq:vltau}
\end{equation}
Note that when $\lambda=0$, $v_{0,\tau}$ is a log-sum-exp smoothing approximation of $\bar{u}$, with the approximation accuracy controlled by $\tau$~\cite{nesterov2005smooth,boyd2004convex}. Introducing $\lambda>0$ adds a quadratic regularization term that improves the curvature of the inner problem, which will be shown to yield a well-conditioned surrogate with stable gradients. 

In the following, we show that $v_{\lambda,\tau}$ preserves the key properties of $\bar{u}$, while quantifying the approximation gap. We begin by showing that $v_{0,\tau}$ is a smooth approximation of $\bar{u}$.
\begin{proposition}[Smooth Approximation]
\label{prop:smooth_approx}
The function $v_{0, \tau}(\theta)$ is a smoothing approximation of $\bar{u}(\theta)$. That is, for any $\tau>0$, $v_{0, \tau}(\theta)$ is continuously differentiable on
$\mathbb{R}^d$ and satisfies the properties:
\begin{enumerate}
    \item \emph{Consistency:}\;
    $\lim_{\tau\to 0} v_{0,\tau}(\theta)=\bar{u}(\theta)$ for all $\theta\in \mathbb{R}^d$;
    \item \emph{Smoothness:}\;
    there exist constants $C>0$ and $\alpha>0$, independent of $\tau$, such that
    $v_{0,\tau}$ has a Lipschitz gradient on $\mathbb{R}^d$ with the Lipschitz constant
    $L_{v_{0,\tau}} = C + \alpha/\tau$.
\end{enumerate}
\end{proposition}

\noindent
In other words, Proposition~\ref{prop:smooth_approx} shows that $v_{0, \tau}(\theta)$ shares similar properties with the transformation-induced merit function $\bar{u}(\theta)$, while being much easier to be optimized via gradient-based methods due to its differentiable nature. Moreover, it is easy to see that $\bar{u}(\theta)$ is properly upper- and lower-bounded by $v_{0,\tau}(\theta)$.
\begin{proposition}[Bounded Approximation]
\label{prop:bounded_approx}
For any $\theta\in \mathbb{R}^d$, we have $v_{0,\tau}(\theta) \le \bar{u}(\theta)\leq v_{0,\tau}(\theta) + \tau\ln m $.
\end{proposition}
Proposition~\ref{prop:bounded_approx} implies that $v_{0,\tau}$ provides a \emph{uniform} smooth approximation to $\bar{u}$ with a gap controlled by $\tau$. In particular, choosing $\tau$ sufficiently small makes the bound tight uniformly over $\mathbb{R}^d$.

Next, we show that the regularization $\lambda\geq 0$ may yield a strongly-convex inner structure in problem~\eqref{eq:vltau}, which ensures that the inner minimizer is well-defined (i.e., it exists and is unique) and facilitates the first-order analysis. We first make the weak convexity assumption, which is widely used in nonconvex optimization~\cite{lin2020gradient,xu2023unified,chen2025efficient}.
\begin{assumption}[Weak convexity]
\label{ass:weak_convexity}
For each $i\in[m]$, the function $\tilde{\mathcal{L}}_i:\mathbb{R}^d \to\mathbb{R}$ is assumed to be $\mu_i$-weakly convex on $\mathbb{R}^d$.
That is, $\tilde{\mathcal{L}}_i(\theta)+\frac{\mu_i}{2}\|\theta\|^2$ is convex on $\mathbb{R}^d$.
\end{assumption}

\begin{lemma}[Uniqueness of the inner minimizer]
\label{lemmsc}
Suppose that Assumption~\ref{ass:weak_convexity} holds. If $\lambda>\bar{\mu}$, where $\bar{\mu}:=\max_{i\in[m]}\mu_i$, then for any fixed $\theta\in \mathbb{R}^d$, function $h_{\lambda,\tau}(\theta,\theta')$ is strictly convex with respect to $\theta'$. 
Consequently, the inner problem $\min_{\theta'\in \mathbb{R}^d} h_{\lambda,\tau}(\theta,\theta')$ admits a unique minimizer, denoted by $
\theta^{\star}_{\lambda,\tau}(\theta)
=
\arg\min_{\theta'\in\mathbb{R}^d} h_{\lambda,\tau}(\theta,\theta')$.
\end{lemma}

Lemma~\ref{lemmsc} shows that when $\lambda > \bar{\mu}$, the inner problem becomes uniformly strongly-convex, which guarantees a unique minimizer and enables stable gradient-based updates with linear convergence~\cite{boyd2004convex}. Moreover, through the implicit minimizer mapping $\theta^{\star}_{\lambda,\tau}(\theta)$, the outer objective $v_{\lambda,\tau}$ inherits favorable smoothness properties. Therefore, we will impose the condition $\lambda > \bar{\mu}$ throughout the remainder of the analysis.

Finally, we show that the surrogate $v_{\lambda,\tau}(\theta)$ preserves the weak Pareto optimal property up to a controllable tolerance.

\begin{proposition}
\label{theov}
Suppose that Assumption~\ref{ass:weak_convexity} holds. Then $v_{\lambda,\tau}(\theta)$ satisfies the following properties :
\begin{enumerate}
    \item (Lower bound) $v_{\lambda,\tau}(\theta) \ge -\tau \ln m$ for any $\theta \in \mathbb{R}^d$.
    \item (Necessity) If $\theta$ is weakly Pareto optimal for problem~\eqref{eq:1}, then $v_{\lambda,\tau}(\theta) \le 0$.
    \item (Approximate sufficiency) If $\lambda = 0$ and $v_{0,\tau}(\theta) \le 0$, then $\theta$ is $\epsilon$-weakly Pareto optimal with $\epsilon = \tau \ln m$.
    \item (Sufficiency under convexity) If $\lambda > 0$, $v_{\lambda,\tau}(\theta) \le -\tau \ln m$, and each $\mathcal{L}_i$ is convex at $\theta$, then $\theta$ is weakly Pareto optimal.
\end{enumerate}
\end{proposition}
When $\lambda = 0$, the smoothing introduces at most an $\epsilon = \tau \ln m$ error\footnote{***an approximation error of $\epsilon = \tau \ln m$? \edit{Done} Yes, relaxation changes error}, meaning that minimizing $v_{0,\tau}$ yields an $\epsilon$-weakly Pareto optimal solution to the original problem. This result justifies $v_{\lambda,\tau}$ as a principled surrogate for $\bar{u}(\theta)$ with a clear trade-off between smoothness and Pareto optimality gap\footnote{***what does Pareto accuracy mean?\edit{Done} changes to optimality gap} controlled by $\tau$.

\subsection{Algorithm and Convergence Analysis}\label{analysis}
To solve the saddle-point problem~\eqref{eq:vltau}, 
we adopt the two-time-scale gradient descent-ascent (TTGDA) method~\cite{lin2020gradient,xu2023unified}, where the inner variable $\theta'$ is updated with a larger learning rate to rapidly track the inner minimizer, while the outer variable $\theta$ is updated more conservatively to ensure stable descent. 
Algorithm~\ref{alg1} summarizes the TTGDA method for solving problem~\eqref{eq:vltau}.

To compute the gradients, we define the weights of gradient aggregation as
\[
w_i(\theta,\theta') :=
\frac{\exp\!\left(\frac{\tilde{\mathcal{L}}_i(\theta')-\tilde{\mathcal{L}}_i(\theta)}{\tau}\right)}
{\sum_{j=1}^m \exp\!\left(\frac{\tilde{\mathcal{L}}_j(\theta')-\tilde{\mathcal{L}}_j(\theta)}{\tau}\right)}, 
\quad i\in[m].
\]
These weights arise directly from the log-sum-exp smoothing and form a probability simplex. They emphasize tasks whose improvement from $\theta$ to $\theta'$ is relatively small, thereby adaptively reweighting task gradients based on local changes in the objective values.\footnote{***what does `local first-order information' mean? \edit{Done} changes to local changes in the objective values}. 
Using these weights, the gradients of $h_{\lambda,\tau}(\theta,\theta')$ with respect to $\theta$ and $\theta'$ are computed as
\begin{align}
\nabla_{\theta'} h_{\lambda,\tau}(\theta,\theta')
&=
\sum_{i=1}^m w_i(\theta,\theta')\nabla \tilde{\mathcal{L}}_i(\theta')
+
\lambda(\theta'-\theta),
\label{eq:grad_1}
\\\!
\nabla_{\theta} h_{\lambda,\tau}(\theta,\theta')
&=\!
-\sum_{i=1}^m w_i(\theta,\theta')\!\nabla \tilde{\mathcal{L}}_i(\theta)
\!+\! \lambda(\theta-\theta').
\label{eq:grad_2}
\end{align}

\begin{algorithm}[t]
\caption{TTGDA Algorithm}
\label{alg1}
\begin{algorithmic}[1]
\Require Iterations $K$, smoothing parameters $\tau>0$, $\lambda\ge 0$,
stepsizes $\eta_{\theta'} \gg \eta_{\theta}$,
initial points $\theta^{0},\theta'^{0}\in\mathbb{R}^d$.
\For{$k=0,1,\ldots,K-1$}
\State Compute gradients $\nabla_{\theta'} h_{\lambda,\tau}(\theta^{k},\!\theta'^{k})$ and $\nabla_{\theta} h_{\lambda,\tau}(\theta^{k},\!\theta'^{k})$ via Eqs.~\eqref{eq:grad_1} and~\eqref{eq:grad_2}
    \State Update
    \begin{align}
        &\theta'^{k+1} \leftarrow \theta'^{k} - \eta_{\theta'}\, \nabla_{\theta'} h_{\lambda,\tau}\!\left(\theta^{k},\theta'^{k}\right),\notag \\
        &\theta^{k+1}\leftarrow \theta^{k} + \eta_{\theta}\, \nabla_{\theta} h_{\lambda,\tau}\!\left(\theta^{k},\theta'^{k}\right). \notag
    \end{align}
\EndFor
\State Randomly draw $\hat{\theta}$ uniformly from $\{\theta^k\}_{k=0}^{K}$.
\Ensure $\hat{\theta}$.
\end{algorithmic}
\end{algorithm}

Next, we will analyze the convergence of the TTGDA algorithm. We begin with some basic assumptions.

\begin{assumption}[Lipschitz continuity]
\label{ass:lipschitz}
For each $i\in[m]$, the function $\tilde{\mathcal{L}}_i:\mathbb{R}^d \to\mathbb{R}$ is
$\ell_i$-Lipschitz continuous on $\mathbb{R}^d$. That is, for all
$\theta_1,\theta_2\in \mathbb{R}^d$, $\big|\tilde{\mathcal{L}}_i(\theta_1)-\tilde{\mathcal{L}}_i(\theta_2)
\big|\le \ell_i \|\theta_1-\theta_2\|$.
\end{assumption}

\begin{assumption}[Smoothness]
\label{ass:smoothness}
For each $i\in[m]$, the function $\tilde{\mathcal{L}}_i$ is $\zeta_i$-smooth on $\mathbb{R}^d$. That is, for all $\theta_1,\theta_2\in \mathbb{R}^d$, $\|\nabla \tilde{\mathcal{L}}_i(\theta_1)-\nabla \tilde{\mathcal{L}}_i(\theta_2)\|
\le \zeta_i \|\theta_1-\theta_2\|$.
\end{assumption}

The Lipschitz continuity and smoothness assumptions in Assumptions~\ref{ass:lipschitz} and~\ref{ass:smoothness} are widely adopted in MOO~\cite{ sener2018multi,lin2024smooth,chen2025gradient}. Based on these assumptions, we can characterize the smoothness of the saddle function $h_{\lambda,\tau}(\theta,\theta')$.
\begin{lemma}\label{smoothnessh}
Under Assumptions~\ref{ass:lipschitz} and~\ref{ass:smoothness}, define $\bar{\ell}:=\max_{i\in[m]}\ell_i$ and $\bar{\zeta}:=\max_{i\in[m]}\zeta_i$. Then $h_{\lambda,\tau}(\theta,\theta')$ is $\ell$-smooth jointly in $(\theta,\theta')$, with $\ell=\frac{\bar{\ell}^2}{\tau}+\bar{\zeta}+\lambda$.
\end{lemma}

\noindent
We define $\Delta_v$ and $\delta^0$ as
\[
\Delta_v =v_{\lambda,\tau}(\theta^0) - \min_{\theta\in\mathbb{R}^d} v_{\lambda,\tau}(\theta), \quad \delta^0:=\bigl\|\theta^{\star}_{\lambda,\tau}(\theta^0)-\theta'^0\bigr\|^2.
\]
Then in the following theorem, we analyze the convergence property of Algorithm~\ref{alg1} for solving problem~\eqref{eq:vltau}.

\begin{theorem}
\label{thm:gda_complexity}
Under Assumptions~\ref{ass:weak_convexity},~\ref{ass:lipschitz} and~\ref{ass:smoothness}, we set  $\lambda>\bar{\mu}$, $\eta_{\theta'}=\Theta(1/\ell)$, and $\eta_{\theta}=\Theta(1/(\kappa^2\ell))$, where $\kappa=\frac{\ell}{\lambda-\bar{\mu}}>0$ denote the condition number~\cite{lin2020gradient}. The iteration complexity (also the gradient complexity) of Algorithm~\ref{alg1} to obtain an $\epsilon$-stationary point is bounded by
\[
\mathcal{O}\!\left(
\frac{\kappa^2\ell\,\Delta_v
+
\kappa\ell^2 (\delta^0)^2}{\epsilon^2}
\right).
\]
\end{theorem}

Theorem~\ref{thm:gda_complexity} establishes the convergence guarantee of TTGDA for problem~\eqref{eq:vltau}. The condition number $\kappa$ captures the intrinsic asymmetry between the outer and inner variables and leads to a two-time-scale stepsize separation $\eta_{\theta'}/\eta_{\theta}=\Theta(\kappa^2)$. Such a separation is necessary for stability when using gradient methods in nonconvex-strongly-convex minimax optimization~\cite{lin2020gradient,xu2023unified,zhang2024generalization}.\footnote{***any technical innovation in this analysis? if yes, highlight it.}


\section{Experiments}
In this section, we evaluate SIMS on three standard multi-task benchmarks
to demonstrate its effectiveness across diverse task compositions and objective scales. Furthermore, we conduct ablation studies on the smoothing parameter $\tau$ and the impact of the regularization parameter $\lambda$. In addition, we investigate different transformation functions $\psi$ and demonstrate the superiority of SIMS over heuristic loss normalization strategies.

\subsection{Experimental Setup}
\noindent \textbf{Datasets.}
The following datasets are used: (i) NYUv2~\cite{silberman2012indoor}, which is an indoor
scene understanding dataset. It has 3 tasks (13-class semantic segmentation, depth estimation, and
surface normal prediction) with 795 training and 654 testing images. 
(ii) CityScapes~\cite{cordts2016cityscapes}, which is an urban scene understanding dataset. It has 2 tasks (7-class semantic segmentation and depth estimation) with 2,975 training and 500 testing images.
And (iii) PASCAL-Context~\cite{everingham2010pascal}, which is a challenging scene understanding dataset for images in the wild. We conduct experiments on 4 tasks (21-class semantic segmentation, 7-class human parts segmentation, saliency estimation, and surface normal estimation) with 4,998 training and 5,105 testing images. These benchmarks were selected for their diverse task definitions and supervision statistics. This diversity induces large disparities in objective scales, serving as a rigorous setting to demonstrate the scale-invariant property of our proposed method.
\vspace{0.5\baselineskip}

\noindent \textbf{Baselines.}
We compare the proposed SIMS with other scalarization methods, including Equal Weight (EW), STCH~\cite{lin2024smooth}, and FOOPS~\cite{chen2025efficient}. To assess competitiveness against state-of-the-art multi-task architectures, we also compare with representative methods across three categories. For CNN-based methods, we include Hard-Parameter Sharing (HPS), employing a shared backbone with task-specific output heads, Cross-Stitch~\cite{misra2016cross}, Multi-Task Attention Network (MTAN)~\cite{liu2019end}, and NDDR-CNN~\cite{gao2019nddr}. Regarding Transformer-based methods, we include VTAGML~\cite{bhattacharjee2023vision}, SwinMTL~\cite{taghavi2024swinmtl}, and DenseMTL \cite{lopes2023cross}. Finally, we consider foundation models utilizing parameter-efficient fine-tuning  \cite{hu2022lora}, 
such as MultiLoRA~\cite{wang2023multilora} and MTSAM~\cite{wang2025mtsam}.
\vspace{0.5\baselineskip}

\noindent \textbf{Evaluation metric.}
We report the common evaluation metrics for each task and provide comprehensive details in Appendix~\ref{appdix:metric}. Following the experimental setup in~\cite{maninis2019attentive},  we use the average of the relative improvement ($\Delta_b$) of each task over the baseline method (i.e., HPS) to quantify overall performance, which is formulated as $\Delta_b = \frac{1}{T} \sum_{i=1}^{T} \frac{1}{K_i} \sum_{j=1}^{K_i}
\frac{(-1)^{s_{i,j}} \left( M_{i,j}^{p} - M_{i,j}^{\text{base}} \right)}
{M_{i,j}^{\text{base}}}$, where 
$T$ denotes the number of tasks and $K_i$ is the number of evaluation metrics associated with task $i$.
$M_{i,j}^{p}$ and $M_{i,j}^{\text{base}}$ represent the performance of method $p$ and the baseline method on the $j$-th metric of task $i$, respectively.
The indicator $s_{i,j} \in \{0,1\}$ specifies the optimization direction of each metric, where $s_{i,j}=1$ indicates that lower values correspond to better performance, and $s_{i,j}=0$ otherwise.
\vspace{0.5\baselineskip}

\noindent \textbf{Model.}
We adopt the paradigm of adapting pretrained backbones via PEFT \cite{hu2022lora,liu2022polyhistor}. 
Specifically, we utilize the image encoder from SAM2.1-Large \cite{ravi2024sam} as the foundational feature extractor. To adapt this frozen backbone for MTL, we employ vanilla Low-Rank Adaptation (LoRA) \cite{hu2022lora}. Formally, LoRA approximates the weight update $\Delta W$ for a pretrained weight matrix $W_0 \in \mathbb{R}^{d \times k}$ as the product of two low-rank matrices $B \in \mathbb{R}^{d \times r}$ and $A \in \mathbb{R}^{r \times k}$, where the rank $r \ll \min(d, k)$.
The forward pass is thus modified as $h = W_0 x + s\Delta W x = W_0 x + sB A x$, where $s$ is a scaling coefficient.
We inject these shared LoRA adapters into the Query, Key, and Value (QKV) projections of the transformer layers to learn shared representations across different downstream tasks while keeping most parameters frozen. Finally, following prior work~\cite{wang2025mtsam}, we employ a modified lightweight prompt encoder to decode these features for specific outputs.

\begin{table*}[!phbt]
\centering

\setlength{\tabcolsep}{6pt}
\renewcommand{\arraystretch}{1}
\caption{Performance on three tasks on the NYUv2 dataset. 
Results marked with $^{\dagger}$ are copied from prior work. The best results for each task are shown in \textbf{bold}, and the second best are underlined.
$\uparrow(\downarrow)$ indicates that higher (lower) is better.}
\vspace{-3pt}
\label{table:nyu}
\begin{tabular}{lcccccccccc}
\toprule
\multirow{3}{*}[-1ex]{\textbf{Method}}
& \multicolumn{2}{c}{\textbf{Segmentation}}
& \multicolumn{2}{c}{\textbf{Depth}}
& \multicolumn{5}{c}{\textbf{Surface Normal}}
& \multirow{3}{*}[-1ex]{$\Delta_b \uparrow$} \\
\cmidrule(lr){2-3}\cmidrule(lr){4-5}\cmidrule(lr){6-10}
& \multirow{2}{*}[-0.4ex]{\textbf{mIoU}$\uparrow$}
& \multirow{2}{*}[-0.4ex]{\textbf{Pix Acc}$\uparrow$}
& \multirow{2}{*}[-0.4ex]{\textbf{Abs Err}$\downarrow$}
& \multirow{2}{*}[-0.4ex]{\textbf{Rel Err}$\downarrow$}
& \multicolumn{2}{c}{\textbf{Angle Distance}}
& \multicolumn{3}{c}{\textbf{Within $t^\circ$}}
& \\
\cmidrule(lr){6-7}\cmidrule(lr){8-10}
& & & & 
& \textbf{Mean}$\downarrow$
& \textbf{Median}$\downarrow$
& \textbf{11.25}$\uparrow$
& \textbf{22.5}$\uparrow$
& \textbf{30}$\uparrow$
& \\
\midrule
HPS$^{\dagger}$          & 54.48 & 75.82 & 0.3839 & 0.1548 & 23.50 & 17.06 & 35.31 & 61.10 & 72.14 & +0.00\% \\
Cross-Stitch$^{\dagger}$ & 53.46 & 75.49 & 0.3804 & 0.1555 & 23.01 & 16.33 & 37.01 & 62.42 & 73.02 & +0.66\% \\
MTAN$^{\dagger}$         & 54.74 & 75.78 & 0.3796 & 0.1549 & 22.97 & 16.30 & 36.91 & 62.63 & 73.32 & +0.77\% \\
NDDR-CNN$^{\dagger}$     & 53.84 & 75.23 & 0.3871 & 0.1560 & 22.60 & 16.07 & 37.67 & 63.43 & 73.92 & +0.91\% \\
VTAGML$^{\dagger}$       & 58.60 & 78.63 & 0.3716 & 0.1525 & 22.05 & 15.70 & 38.14 & 64.28 & 74.50 & +4.70\% \\
DenseMTL$^{\dagger}$     & 56.65 & 77.68 & 0.3569 & 0.1391 & 22.03 & 15.87 & 37.25 & 64.67 & 75.47 & +5.88\% \\
SwinMTL$^{\dagger}$      & 64.23 & 82.78 & 0.2841 & 0.1129 & 18.94 & 13.34 & 43.35 & 71.32 & 80.89 & +19.55\% \\
MultiLoRA$^{\dagger}$        & 64.85 & 83.07 & 0.3113 & 0.1220 & 17.26 & 12.19 & 48.28 & 74.65 & 83.58 & +20.11\% \\
MTSAM$^{\dagger}$ & 65.98 & 83.42 & 0.2898 & 0.1140 & \textbf{16.34} & \textbf{11.33} & \textbf{51.22} & \textbf{77.20} & \textbf{85.51} & +23.93\% \\
\midrule
EW    & 66.64 & 84.21 & 0.2817 & 0.1107 & 17.86 & 12.66 & 45.64 & 73.39 & 82.80 & +22.34\% \\
GLS   & 66.80 & 84.44 & 0.2762 & 0.1084 & 17.23 & 12.12 & 47.50 & 74.96 & 84.03 & \underline{+23.95\%} \\
STCH  & 66.27 & 83.87 & \underline{0.2742} & \underline{0.1077} & 17.25 & 12.12 & 47.61 & 74.79 & 84.88 & +23.81\% \\
FOOPS & \textbf{67.50} & \textbf{84.80} & 0.2787 & 0.1096 & 17.60 & 12.42 & 46.62 & 74.07 & 83.30 & +23.46\% \\
SIMS (Ours)  & \underline{67.12} & \underline{84.70} & \textbf{0.2722} & \textbf{0.1050} & \underline{16.85}   & \underline{11.65}   & \underline{49.35} & \underline{75.85} & \underline{84.54} & +\textbf{25.44\%} \\
\bottomrule
\end{tabular}
\vspace{-6pt}
\end{table*}

\vspace{0.5\baselineskip}

\noindent \textbf{Implementation Details.}
We optimize the model using Adam with a learning rate of $1\times10^{-4}$ and weight decay of $1\times10^{-6}$, while applying a linear warm-up for the first $10\%$ of training steps. We maintain a consistent LoRA configuration across these datasets \footnote{***cy: Check learning rate here. Maybe delete `and learning rate'. `main datasets'? Maybe delete `the main'. \edit{Done}} by setting the rank to 32, the dropout rate to 0.1, and the scaling coefficient to 0.5. Models are trained for 100 epochs on NYU and Cityscapes, and 30 epochs on PASCAL-Context. The batch size of all datasets is 4. For computational efficiency, we use the SAM2.1-Tiny backbone on NYUv2 and accordingly reduce the LoRA rank to 16 while keeping all other settings unchanged in Sec.~\ref{subsec:trans} and Sec.~\ref{subsec:prenorm}. More details are provided in Appendix~\ref{app:mtldetails}.

We follow the configurations reported in the original papers for all baseline methods. In particular, for STCH, we adhere to the default setting that applying objective normalization proposed by~\cite{dai2023improvable} to alleviate the mismatch in objective scales across tasks. In addition, SIMS, STCH, and FOOPS all utilize log-sum-exp smoothing. Since~\cite{chen2025efficient} and our subsequent ablation studies indicate that performance is relatively insensitive to the smoothing temperature within a reasonable range, we use a shared smoothing parameter $\tau$ across these methods to ensure a fair comparison.

\begin{table}[!pbht]
\centering
\setlength{\tabcolsep}{2pt}
\renewcommand{\arraystretch}{1}
\caption{Performance on two tasks on the CityScapes dataset.
Results marked with $^{\dagger}$ are copied from prior work. The best results for each task are shown in \textbf{bold}, and the second best are underlined.
$\uparrow(\downarrow)$ indicates that higher (lower) is better.}
\vspace{-4pt}
\label{tab:cityscapes_results}
\resizebox{1.0\columnwidth}{!}{
\begin{tabular}{lcccccc}
\toprule
\multirow{2}{*}{\textbf{Method}} 
& \multicolumn{2}{c}{\textbf{Segmentation}} 
& \multicolumn{2}{c}{\textbf{Depth}} 
& \multirow{2}{*}{$\Delta_b \uparrow$} \\
\cmidrule(lr){2-3}\cmidrule(lr){4-5}
& \textbf{mIoU}$\uparrow$ & \textbf{Pix Acc}$\uparrow$ & \textbf{Abs Err}$\downarrow$ & \textbf{Rel Err}$\downarrow$ &  \\
\midrule
HPS$^{\dagger}$          & 67.40 & 90.92 & 0.0142 & 45.4262 & +0.00\% \\
STL$^{\dagger}$          & 68.13 & 91.28 & 0.0133 & 45.0390 & +2.17\% \\
Cross-Stitch$^{\dagger}$ & 68.01 & 91.29 & 0.0135 & 44.4246 & +2.11\% \\
MTAN$^{\dagger}$         & 68.97 & 91.59 & 0.0136 & 43.7508 & +2.74\% \\
NDDR-CNN$^{\dagger}$     & 68.02 & 91.25 & 0.0137 & 44.8662 & +1.51\% \\
VTAGML$^{\dagger}$       & 73.70 & 93.23 & 0.0138 & 42.8304 & +5.10\% \\
DenseMTL$^{\dagger}$     & 69.75 & 91.45 & 0.0152 & 52.1401 & -4.44\% \\
SwinMTL$^{\dagger}$      & 73.33 & 92.87 & 0.0132 & 38.0720 & +8.54\% \\
MultiLoRA$^{\dagger}$        & 87.23 & 96.67 & 0.0127 & 31.0091 & +19.51\% \\
MTSAM$^{\dagger}$ & 87.45 & 96.80 & 0.0113 & 33.0086 & +20.99\% \\
\midrule
EW & \textbf{88.24} & \textbf{97.02} & 0.0097 & 41.3157 & +19.59\% \\
GLS & 87.79 & 96.88 & \underline{0.0085} & 25.9219 & +29.97\% \\
STCH  & 88.07 & 96.95 & 0.0088 & \textbf{24.5415} & \underline{+30.33\%} \\
FOOPS & \underline{88.21} & \underline{97.00} & 0.0096 & 38.4143 & +21.35\% \\
SIMS (Ours)  & 87.91 & 96.90 & \textbf{0.0084} & \underline{24.8550} & +\textbf{30.78\%} \\
\bottomrule
\end{tabular}
}
\vspace{-10pt}
\end{table}

\begin{table}[!phbt]
\centering
\setlength{\tabcolsep}{2pt}
\renewcommand{\arraystretch}{1}
\caption{Performance on four tasks on the PASCAL-Context dataset. 
Results marked with $^{\dagger}$ are from prior work. The best results for each task are shown in \textbf{bold}, and the second best are underlined.
$\uparrow(\downarrow)$ indicates that higher (lower) is better.}
\vspace{-4pt}
\label{tab:pascal_context_results}
\resizebox{0.9\columnwidth}{!}{
\begin{tabular}{lccccc}
\toprule
\textbf{Method} & \textbf{Seg.}$\uparrow$ & \textbf{H.Parts}$\uparrow$ & \textbf{Sal.}$\uparrow$ & \textbf{Normal}$\downarrow$ & $\Delta_b \uparrow$ \\
\midrule
HPS$^{\dagger}$          & 64.77 & 57.91 & 64.10 & \textbf{14.21} & +0.00\% \\
STL$^{\dagger}$          & 65.14 & 58.58 & 65.02 & 15.94 & -2.25\% \\
Cross-Stitch$^{\dagger}$ & 64.97 & 58.63 & 64.46 & 15.32 & -1.42\% \\
MTAN$^{\dagger}$         & 64.56 & 59.08 & 64.57 & 14.74 & -0.33\% \\
NDDR-CNN$^{\dagger}$     & 65.28 & 59.18 & 65.09 & 15.57 & -1.26\% \\
MultiLoRA$^{\dagger}$    & 72.39 & 67.78 & 71.66 & 20.07 & -0.16\% \\
MTSAM$^{\dagger}$        & 74.13 & 71.04 & 76.28 & 17.10 & \underline{+8.95\%} \\
\midrule
EW                      & 74.83 & 71.92 & 74.36 & 18.42 & +6.53\% \\
GLS                     & 76.36 & \textbf{74.92} & 76.22 & 18.77 & +8.54\% \\
FOOPS                    & \underline{77.56} & 73.56 & 74.78 & \underline{18.17} & +8.89\% \\
STCH                     & 73.66 & 73.18 & \underline{76.97} & 18.18 & +8.06\% \\
SIMS (Ours)        & \textbf{77.60} & \underline{73.88} & \textbf{77.56} & 18.22 & +\textbf{10.04\%} \\
\bottomrule
\end{tabular}
}
\vspace{-6pt}
\end{table}




\begin{table*}[!pbht]
\centering
\setlength{\tabcolsep}{6pt}
\renewcommand{\arraystretch}{1}
\caption{Comparison of transformation functions and loss normalization strategies. Performance on three tasks on the NYUv2 dataset. 
The best results are shown in \textbf{bold}, and the second best are underlined.
$\uparrow(\downarrow)$ indicates that higher (lower) is better.}
\vspace{-3pt}
\label{table:trans}
\resizebox{1.85\columnwidth}{!}{
\begin{tabular}{ccccccccccc}
\toprule
\multirow{3}{*}[-1ex]{\textbf{$\psi$}}
& \multicolumn{2}{c}{\textbf{Segmentation}}
& \multicolumn{2}{c}{\textbf{Depth}}
& \multicolumn{5}{c}{\textbf{Surface Normal}}
& \multirow{3}{*}[-1ex]{$\Delta_b \uparrow$} \\
\cmidrule(lr){2-3}\cmidrule(lr){4-5}\cmidrule(lr){6-10}
& \multirow{2}{*}[-0.4ex]{\textbf{mIoU}$\uparrow$}
& \multirow{2}{*}[-0.4ex]{\textbf{Pix Acc}$\uparrow$}
& \multirow{2}{*}[-0.4ex]{\textbf{Abs Err}$\downarrow$}
& \multirow{2}{*}[-0.4ex]{\textbf{Rel Err}$\downarrow$}
& \multicolumn{2}{c}{\textbf{Angle Distance}}
& \multicolumn{3}{c}{\textbf{Within $t^\circ$}}
& \\
\cmidrule(lr){6-7}\cmidrule(lr){8-10}
& & & & 
& \textbf{Mean}$\downarrow$
& \textbf{Median}$\downarrow$
& \textbf{11.25}$\uparrow$
& \textbf{22.5}$\uparrow$
& \textbf{30}$\uparrow$
& \\
\midrule
\multicolumn{11}{c}{Our Method} \\
$\ln(x)$      & \underline{54.57} & 76.72 & \textbf{0.3614} & 0.1476 & \textbf{20.22} & \textbf{14.03} & \textbf{41.48} & \textbf{68.29} & \textbf{78.19} & \textbf{0.00\%} \\
\midrule
\multicolumn{11}{c}{Different Transformation} \\
$x^2$      & 53.15 & 76.26 & 0.3803 & 0.1548 & 21.99 & 15.79 & 37.07 & 63.85 & 74.57 & -5.09\% \\
$x$        & 54.12 & 76.68 & 0.3726 & 0.1519 & 21.24 & 15.03 & 38.90 & 65.76 & 76.12 & -2.80\% \\
$\sqrt{x}$     & 54.21 & 76.80 & 0.3663 & 0.1498 & 20.86 & 14.67 & 39.94 & 66.66 & 76.85 & -1.61\% \\
$asinh(\frac{x}{0.1})$ & 53.75 & \underline{76.81} & \underline{0.3631} & 0.1479 & 20.48 & 14.29 & 40.87 & 67.66 & 77.68 & -0.75\% \\
$asinh(\frac{x}{0.01})$& \textbf{54.99} & \textbf{77.06} & 0.3643 & 0.1478 & \underline{20.29} & 14.12 & 41.26 & \underline{68.18} & \underline{78.12} & \underline{-0.07\%} \\
\midrule
\multicolumn{11}{c}{Loss Normalization} \\
Fix@Epoch 2& 53.72 & 76.64 & 0.3639 & 0.1479 & 20.63 & 14.55 & 40.14 & 67.20 & 77.38 & -1.20\% \\
Epoch-wise&53.94 & 76.39 & \underline{0.3631} & \textbf{0.1472} & 20.25 & \underline{14.10} & \underline{41.36} & 68.15 & \underline{78.12} & -0.38\% \\
EMA      & 53.98 & 76.79 & 0.3645 & \underline{0.1475} & 20.65 & 14.60 & 39.92 & 67.07 & 77.32 & -1.15\% \\

\bottomrule
\end{tabular}
}
\vspace{-3pt}
\end{table*}

\subsection{Main Results}

Our results on NYUv2, Cityscapes, and PASCAL-Context are reported in Tables~\ref{table:nyu},~\ref{tab:cityscapes_results} and~\ref{tab:pascal_context_results}.\footnote{***cy: It would be better to underline the second-best result since you are always the second-best result in Table 1. and then explain the underline in the legend. \edit{Done}} For convenience, we directly report some results from prior work~\cite{wang2025mtsam}, which are marked with $^{\dagger}$.
As presented in the results, SIMS achieves state-of-the-art (SOTA) average performance across all three benchmarks. 

Notably, EW consistently yielded the lowest performance among all evaluated scalarization methods. Crucially, the advantage of SIMS extend beyond superior average metrics. The intrinsic scale invariance of SIMS\footnote{***cy: check `scale-invariant formulation' \edit{Done}} enables it to yield a balanced performance profile that avoids the extreme trade-offs where performance on specific task is severely compromised. This superiority is particularly pronounced on the Cityscapes dataset as shown in Table~\ref{tab:cityscapes_results}, where we observe that FOOPS converges to solutions heavily dominated by the segmentation objective, neglecting depth estimation due to the scales of losses. While STCH mitigates this via explicit objective normalization and attains results comparable to SIMS, such heuristic alignment relies on statistical estimation that may face scalability challenges. As the number of tasks increases (Tables~\ref{table:nyu} and ~\ref{tab:pascal_context_results}), these normalization estimates may become increasingly unstable or inaccurate. In contrast, SIMS \edit{exhibits intrinsic scale invariance, eliminating the need for explicit alignment of the loss magnitudes.}\footnote{***cy: check robustness \edit{Done}} 
These quantitative gains are further corroborated by the visualization in Fig.~\ref{fig:toy-scale}, where SIMS maintains consistent optimization trajectories invariant to objective scaling, thereby avoiding the extreme trade-off solutions that plague baseline methods.

\begin{figure}[!tbph]
    \centering
    \vspace{-5pt}
    \includegraphics[width=1\linewidth]{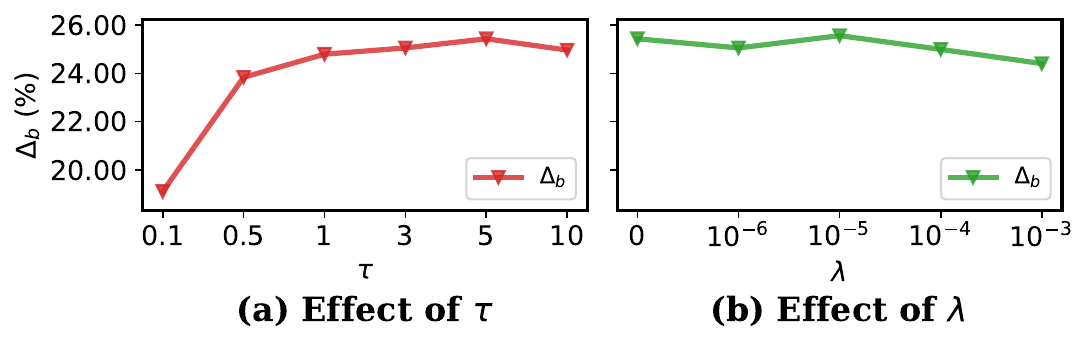}
    \captionsetup{skip=2pt}
    \caption{Ablation study on the hyper-parameters $\tau$ and $\lambda$.}
    \label{fig:ablation_tau_lambda}
    \vspace{-10pt}
\end{figure}

\subsection{Ablation Studies}
\label{sec:ablation}
We conduct ablation studies on the impact of the smoothing parameter $\tau$ and the regularization parameter $\lambda$.
\vspace{0.5\baselineskip}

\noindent \textbf{Effect of $\tau$.}\footnote{***cy: Why `Sensitivity with respect to X' but not `Effect of X' \edit{Done}}
We perform a sensitivity study on the parameter $\tau$ to examine how varying $\tau$ affects the performance of our method, with all other hyper-parameters fixed to the same settings as in previous experiments. According to the results shown in Fig.~\ref{fig:ablation_tau_lambda}(a), we can see that within a reasonable range (i.e., $[1, 10]$)\footnote{***cy: complete it. \edit{Done}} of $\tau$, our method achieves consistently stable performance, achieving competitive or state-of-the-art overall performance. The full results can be found in Appendix~\ref{appdix:tau}.
\vspace{0.8\baselineskip}

\noindent \textbf{Effect of $\lambda$. }
Recall that our method optimizes the smoothed surrogate objective
$v_{\lambda, \tau}(\theta)=-
    \min_{\theta' \in \mathbb{R}^d}
    h_{\lambda,\tau}(\theta,\theta'),$
where $\lambda \ge 0$ weights the proximal regularization term. This term improves the curvature of the inner minimization and can yield better-conditioned gradients in practice. According to Proposition~\ref{theov}, $\lambda$ also affects the Pareto interpretation of the surrogate: when $\lambda=0$, $v_{0, \tau}(\theta)\le 0$ implies an $\epsilon$-weak Pareto guarantee, whereas $\lambda > 0$ can recover weak Pareto optimality under additional mild conditions. As shown in Fig.~\ref{fig:ablation_tau_lambda}(b), we find SIMS to be insensitive to variations in $\lambda$ across several orders of magnitude, with small positive values sometimes providing modest gains. Nevertheless, to keep the approach tuning-light and to use a single uniform setting across all benchmarks, we set $\lambda =0$ in all experiments. The full results can be found in Appendix~\ref{appdix:lambda}.

\vspace{-5pt}

\subsection{Effect of Transformation Functions}
\label{subsec:trans}
Beyond logarithmic transformations, we further investigate the optimality of our choice for the transformation function $\psi$ and its impact on the resulting multi-task trade-off. Since the transformations $\psi$ is required to be monotonically increasing, we consider (i) a power function family $\psi(x)=x^a$ with $a\in\{0.5,1,2\}$ and (ii) inverse-hyperbolic-sine transforms $\psi(x)=asinh(cx)$ with $c\in\{0.1,0.01\}$.

As shown in Table~\ref{table:trans}, in the power function family, the expansive setting ($a=2$) results in the poorest performance by aggravating discrepancies of loss scale, whereas the compressive setting ($a=0.5$) offers tangible improvements over the linear baseline ($a=1$).\footnote{***cy: improvements compared to whom. \edit{Done}} Turning to the $\text{asinh}$ family, which serves as a smooth numerical approximation to the logarithm in engineering \cite{lupton1999modified}, the setting $c=0.01$ demonstrates remarkable performance, trailing the log baseline by only $\Delta_b = -0.07\%$. Overall, our empirical results demonstrate that the logarithmic transformation consistently outperforms all alternative candidates, confirming it as the optimal design choice.

\subsection{Comparison with Loss Normalization}
\label{subsec:prenorm}
We view loss normalization as a heuristic approach to scale invariance. To facilitate a direct comparison with SIMS, this strategy can be interpreted as applying a linear transformation to the task losses, where the associated linear coefficients are estimated from training statistics. We consider a Min–Max normalization baseline inspired by previous works~\cite{dai2023improvable,liu2019end,lakkapragada2023mitigating} that fixes the minimum to zero and varies the estimation scope of the maximum via three schemes: (i) a global estimate set to the mean loss at the second epoch, (ii) an epoch-wise estimate using the last epoch’s mean loss, and (iii) a step-wise estimate computed by an exponential moving average (EMA) over steps.

As shown in Table~\ref{table:trans}\footnote{***cy: As shown in Table XXX \edit{Done}}, all three variants substantially underperform compared to SIMS (i.e., $\Delta_b$ of $-1.20\%$, $-0.38\%$, and $-1.15\%$, respectively), with the epoch-wise variant being the strongest among them. This trend shows that purely statistic-driven normalization\footnote{***cy: `statistic-driven linear rescaling'? check. \edit{Done}} is fragile to the choice of reference, whereas SIMS removes the need for such auxiliary scale estimates by embedding scale invariance directly into the merit function construction.

\begin{figure}[!hpbt]
  \vspace{-5pt}
  \centering  \includegraphics[width=0.88\linewidth]{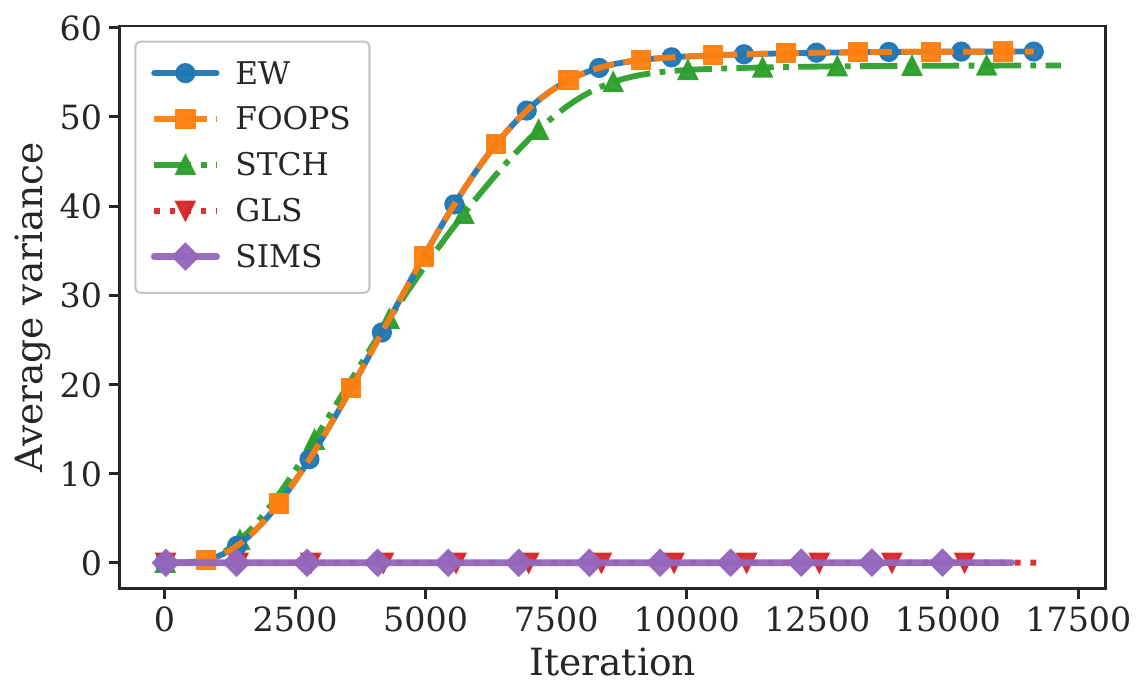}
  \captionsetup{skip=1pt}
  \caption{Trajectory Variance across Scales
  }
  \label{fig:var}
  \vspace{-10pt}
\end{figure}
\begin{add}

\subsection{Synthetic Experiments for Scale Invariance}
To further evaluate whether SIMS is invariant to task-wise loss rescaling, we conduct a quantitative synthetic experiment based on the motivating example in Section~\ref{sec:motivation}.
Here, we examine whether each scalarization method behaves consistently and produces similar trajectories when the relative scales of task losses are artificially changed.
Specifically, we fix the initialization at $(0,0)$ and stop once the distance to the Pareto front is below $0.05$. We evaluate each method under three fixed random seeds and seven rescaling settings: $1:1$, $1:1.25$, $1.25:1$, $1:10$, $10:1$, $100:1$, and $1:100$. These task-wise rescaling settings only change the numerical magnitudes of the objectives but do not alter the Pareto front.

As the quantitative metric, we compute the variance across the seven rescaling settings, where a smaller variance indicates stronger scale-invariant behavior. We analyze the results from both trajectory and final solution perspectives. As shown in Fig.~\ref{fig:var}, the trajectory variances of EW, FOOPS, and STCH increase rapidly during optimization, indicating that their optimization trajectories are strongly affected by task-wise rescaling. In contrast, GLS and SIMS keep the variance close to zero, suggesting nearly consistent trajectories across rescaling settings. At the final solutions, Table~\ref{tab:variance_comparison} shows that SIMS and GLS achieve the smallest variances, while EW, STCH, and FOOPS exhibit much larger variances. These results provide direct empirical evidence that SIMS is substantially more scale-invariant than existing baselines, further supporting the scale-invariant property of the proposed logarithmic transformation-induced merit function.

\end{add}

\begin{table}[t]
\centering
\caption{Variance of final converged solutions under different task-wise rescaling settings. Lower ($\downarrow$) is better.}
\label{tab:variance_comparison}
\begin{tabular}{lccc}
\toprule
Method & Var. of $\ell_1\downarrow$ & Var. of $\ell_2\downarrow$ & Average Var.$\downarrow$ \\
\midrule
EW     & 57.322545 & 57.322545 & 57.322545 \\
FOOPS  & 57.318462 & 57.318462 & 57.318462 \\
STCH   & 50.89209  & 60.6234   & 55.757745 \\
GLS    & 0.020195  & 0.020195  & 0.020195 \\
SIMS (Ours)   & 0.020105  & 0.020105  & 0.020105 \\
\bottomrule
\end{tabular}
\vspace{-10pt}
\end{table}

\section{Conclusion}
In this paper, we address the sensitivity of existing merit-function-based scalarization methods to the scales of task loss in multi-task learning. Specifically, we propose SIMS, a Scale-Invariant Merit-function-based Scalarization method, \edit{which employs a logarithmic transformation-induced merit function to convert the multi-objective problem into a single objective that is insensitive to loss magnitudes}. We theoretically prove that imposing scale invariance uniquely determines the transformation function to be logarithmic. Furthermore, we show that our formulation preserves weak Pareto optimality and admits a smooth surrogate with controllable approximation error. We demonstrate that the resulting surrogate objective can be efficiently solved via the TTGDA algorithm with established convergence guarantees. Extensive experiments on standard benchmarks demonstrate the effectiveness of SIMS. 

\begin{acks}
This work was supported by National Natural Science Foundation of China under Grant no. 62136005 and Shenzhen fundamental research program JCYJ20250604144724032.
\end{acks}

\bibliographystyle{ACM-Reference-Format}
\bibliography{reference}

\clearpage
\appendix

\section{Proofs}
\subsection{Proofs in Section~\ref{section:formulation}}
\begin{proof}[Proof of Theorem~\ref{theo:u0_merit}]
Let $\theta\in \mathbb{R}^d$. By the definition Eq.~\eqref{eq:ubar} of $\bar{u}$, we have
\begin{align*}
\bar{u}(\theta)
&=
\sup_{\theta'\in \mathbb{R}^d}
\min_{i\in[m]}
\bigl\{\psi(\mathcal{L}_i(\theta))-\psi(\mathcal{L}_i(\theta'))\bigr\}\\
&\geq
\min_{i\in[m]}
\bigl\{\psi(\mathcal{L}_i(\theta))-\psi(\mathcal{L}_i(\theta'))\bigr\}
=0, 
\end{align*}
which implies $\bar{u}(\theta)\ge 0$ for all $\theta\in \mathbb{R}^d$.

On the other hand, again by the definition Eq.~\eqref{eq:ubar} of $\bar{u}$, we obtain
\[
\bar{u}(\theta)=0
\;\Longleftrightarrow\;
\min_{i\in[m]}
\bigl\{\psi(\mathcal{L}_i(\theta))-\psi(\mathcal{L}_i(\theta'))\bigr\}\le 0,
\quad \forall\,\theta'\in \mathbb{R}^d.
\]
Hence, there does not exist any $\theta'\in \mathbb{R}^d$ such that
\[
\psi(\mathcal{L}_i(\theta))-\psi(\mathcal{L}_i(\theta'))>0,
\quad \forall\, i\in[m],
\]
or equivalently,
\[
\psi(\mathcal{L}_i(\theta'))<\psi(\mathcal{L}_i(\theta)),
\quad \forall\, i\in[m].
\]
Since $\psi(\cdot)$ a strictly increasing function, we have 
\[
\mathcal{L}_i(\theta')<\mathcal{L}_i(\theta),
\quad \forall\, i\in[m],
\]
which means that $\theta$ is weakly Pareto optimal for problem~\eqref{eq:1} by definition.
\end{proof}

\subsection{Proofs in Section~\ref{sec5}}
\begin{proof}[Proof of Theorem~\ref{thm:log_characterization}]
Fix any task $i$ and consider a special case in which the inner minimum is always attained at $i$ (e.g., other tasks have constant losses). Then the surrogate locally reduces to the difference form
\[
\bar u(\theta)=\psi(x)-\psi(y),
\]
for arbitrary positive values $x=\mathcal{L}_i(\theta)$ and
$y=\mathcal{L}_i(\theta')$. After rescaling by $c>0$, we obtain
\[
\bar u_c(\theta)=\psi(c x)-\psi(c y).
\]
By the assumed invariance property, we have that
$\bar u_c(\theta)-\bar u(\theta)=K(c)$ must be independent of $\theta$. Therefore, for all $x,y,c>0$,
\[
\psi(c x)-\psi(c y)=\psi(x)-\psi(y).
\]
Setting $y=1$ and defining $f(t)=\psi(t)-\psi(1)$ gives
\[
f(c x)=f(x)+f(c),\qquad \forall x,c>0,
\]
which is the multiplicative Cauchy functional equation
$f(xy)=f(x)+f(y)$.
Since $f$ is differentiable, its only solutions are $f(x)=a\ln x$ (~\citet{kuczma2009introduction}). Hence $\psi(x)=a\ln x+b$. Since $\psi$ is strictly increasing, we must have $a>0$.
\end{proof}

\subsection{Auxiliary Lemmas}
\label{app:aux_lemmas}
For convenience, we define the merit function $u_\lambda(\theta)$ and restate
the smoothed merit function $v_{\lambda,\tau}(\theta)$ below.

\begin{equation}
\label{eq:u_lambda}
u_\lambda(\theta) := \sup_{\theta'\in \mathbb{R}^d} \min_{i\in[m]}
\left\{\tilde{\mathcal{L}}_i(\theta)-\tilde{\mathcal{L}}_i(\theta')- \frac{\lambda}{2}\|\theta-\theta'\|^2\right\},
\end{equation}
and
\begin{equation}
\label{eq:v_lambda_tau}
v_{\lambda,\tau}(\theta) :=
-\min_{\theta'\in \mathbb{R}^d}\left\{\tau \ln
\left(\sum_{i=1}^m \exp\!\left(
\frac{\tilde{\mathcal{L}}_i(\theta')-\tilde{\mathcal{L}}_i(\theta)}{\tau}
\right)\right)+\frac{\lambda}{2}\|\theta-\theta'\|^2\right\}.\notag 
\end{equation}


\begin{lemma}[Theorems~3.1 and~3.3 in~\cite{tanabe2024new}]
\label{lem:u_lambda_merit}
For $\lambda\ge 0$, consider the merit function $u_\lambda(\theta)$ defined in Eq.~\eqref{eq:u_lambda}. Then $u_\lambda(\theta)\ge 0$ for all $\theta\in \mathbb{R}^d$. Moreover, $u_0(\theta)=0$ if and only if $\theta$ is weakly Pareto optimal. For $\lambda>0$, if $\theta$ is weakly Pareto optimal, then $u_\lambda(\theta)=0$; furthermore, if $\tilde{\mathcal{L}}_i(\theta)$ is convex for all $i\in[m]$, then $u_\lambda(\theta)=0$ implies that $\theta$ is weakly Pareto optimal.
\end{lemma}


\begin{lemma}[Log-sum-exp preserves weak convexity]
\label{lem:lse_weak_convex}
Let $\tilde{\mathcal{L}}_i(\theta)$, $i\in[m]$, be weakly convex with modulus
$\mu_i\in\mathbb{R}$. Let $\bar{\mu}=\max_{i\in[m]} \mu_i$. Then
\[
\ln\!\left(\sum_{i=1}^m e^{\tilde{\mathcal{L}}_i(\theta)}\right)
\]
is weakly convex with modulus $\bar{\mu}$.
\end{lemma}
\begin{proof}
By definition, since $\bar{\mu}=\max_{i\in[m]} \mu_i$, the function
\[
\tilde{\mathcal{L}}_i(\theta)+\frac{\bar{\mu}}{2}\|\theta\|^2
\]
is convex for all $i\in[m]$.
Because the log-sum-exp function preserves convexity, it follows that
\[
\ln\!\left(
\sum_{i=1}^m
e^{\tilde{\mathcal{L}}_i(\theta)+\frac{\bar{\mu}}{2}\|\theta\|^2}
\right)
\]
is convex.
Rearranging the above expression yields
\begin{align}
\ln\!\left(
\sum_{i=1}^m
e^{\tilde{\mathcal{L}}_i(\theta)+\frac{\bar{\mu}}{2}\|\theta\|^2}
\right)
=
\ln\!\left(\sum_{i=1}^m e^{\tilde{\mathcal{L}}_i(\theta)}\right)
+
\frac{\bar{\mu}}{2}\|\theta\|^2 \notag .
\end{align}
Hence,
\(
\ln\!\left(\sum_{i=1}^m e^{\tilde{\mathcal{L}}_i(\theta)}\right)
\)
is $\bar{\mu}$-weakly convex.
\end{proof}

\begin{lemma}[Smoothness of log-sum-exp smoothing]\label{lem:lse_smoothness}
Suppose that for all $i\in[m]$,
$\tilde{\mathcal{L}}_i$ is $\zeta_i$-smooth and $\ell_i$-Lipschitz continuous on $\mathbb{R}^d$.
Define
\begin{align}
    g_{\mu}(\theta)=\mu\ln\!\Big(\sum_{i=1}^{m}\exp\big(\tilde{\mathcal{L}}_i(\theta)/\mu\big)\Big),
    \quad \mu>0 \notag.
\end{align}
Then $g_{\mu}$ is continuously differentiable on $\mathbb{R}^d$, and its gradient is Lipschitz on $\mathbb{R}^d$ with
\begin{align}
    \|\nabla g_{\mu}(\theta)\!-\!\nabla g_{\mu}(\theta')\|
    \le\!
    \Big(\frac{\bar{\ell}^2}{\mu}+\bar{\zeta}\Big)
    \|\theta-\theta'\|,\, \forall~\theta,\theta'\in \!\mathbb{R}^d \notag.
\end{align} 
Equivalently, $g_{\mu}$ is $\Big(\frac{\bar{\ell}^2}{\mu}+\bar{\zeta}\Big)$-smooth on $\mathbb{R}^d$, where $\bar{\ell}=\max_{i}\ell_i$ and $\bar{\zeta}=\max_{i}\zeta_i$.
\end{lemma}
\begin{proof}
Define $\boldsymbol{\tilde{\mathcal{L}}}(\theta):=(\tilde{\mathcal{L}}_1(\theta),\ldots,\tilde{\mathcal{L}}_m(\theta))^\top$ and
\[
h_\mu(z):=\mu\ln\Big(\sum_{i=1}^m e^{z_i/\mu}\Big),\qquad z\in\mathbb{R}^m,
\]
so that $g_\mu(\theta)=h_\mu(\boldsymbol{\tilde{\mathcal{L}}}(\theta))$. Let $p(z):=\nabla h_\mu(z)$ be the softmax vector, i.e.,
$p_i(z)=\exp(z_i/\mu)/\sum_{j}\exp(z_j/\mu)$, hence $\|p(z)\|_1=1$.
Moreover, $h_\mu$ is $(1/\mu)$-smooth:
\begin{equation}\label{eq:hmu_smooth_icml}
\|\nabla h_\mu(z)-\nabla h_\mu(z')\|\le \frac{1}{\mu}\|z-z'\|,\qquad \forall z,z'\in\mathbb{R}^m .
\end{equation}

By the chain rule, with $J_{\boldsymbol{\tilde{\mathcal{L}}}}(\theta)=(\nabla \tilde{\mathcal{L}}_1(\theta),\ldots,\nabla \tilde{\mathcal{L}}_m(\theta))$,
\[
\nabla g_\mu(\theta)=J_{\boldsymbol{\tilde{\mathcal{L}}}}(\theta)\,\nabla h_\mu(\boldsymbol{\tilde{\mathcal{L}}}(\theta))
=\sum_{i=1}^m p_i(\boldsymbol{\tilde{\mathcal{L}}}(\theta))\,\nabla \tilde{\mathcal{L}}_i(\theta).
\]
For any $\theta,\theta'\in \mathbb{R}^d$, we decompose
\begin{align*}
\|\nabla g_\mu(\theta)-\nabla g_\mu(\theta')\|
\le& \|J_{\boldsymbol{\tilde{\mathcal{L}}}}(\theta)\|\,\|\nabla h_\mu(\boldsymbol{\tilde{\mathcal{L}}}(\theta))-\nabla h_\mu(\boldsymbol{\tilde{\mathcal{L}}}(\theta'))\| \\ 
&+ \|(J_{\boldsymbol{\tilde{\mathcal{L}}}}(\theta)-J_{\boldsymbol{\tilde{\mathcal{L}}}}(\theta'))\,p(\boldsymbol{\tilde{\mathcal{L}}}(\theta'))\|.
\end{align*}

For the first term, From Eq.~\eqref{eq:hmu_smooth_icml}, we have
\[
\|J_{\boldsymbol{\tilde{\mathcal{L}}}}(\theta)\|\,\|\nabla h_\mu(\boldsymbol{\tilde{\mathcal{L}}}(\theta))-\nabla h_\mu(\boldsymbol{\tilde{\mathcal{L}}}(\theta'))\|
\le \frac{\|J_{\boldsymbol{\tilde{\mathcal{L}}}}(\theta)\|}{\mu}\,\|\boldsymbol{\tilde{\mathcal{L}}}(\theta)-\boldsymbol{\tilde{\mathcal{L}}}(\theta')\|.
\]
Since each $\tilde{\mathcal{L}}_i$ is $\ell_i$-Lipschitz on $\mathbb{R}^d$, we have
$\|\nabla \tilde{\mathcal{L}}_i(\theta)\|\le \ell_i$ and $\|\boldsymbol{\tilde{\mathcal{L}}}(\theta)-\boldsymbol{\tilde{\mathcal{L}}}(\theta')\|\le (\max_i \ell_i)\|\theta-\theta'\|$.
Thus $\|J_{\boldsymbol{\tilde{\mathcal{L}}}}(\theta)\|\le \max_i \ell_i$, and the first term is bounded by
$\frac{\max_i \ell_i^2}{\mu}\|\theta-\theta'\|$.

For the second term, using $\|p(\boldsymbol{\tilde{\mathcal{L}}}(\theta'))\|_1=1$ and the $\zeta_i$-smoothness of $\tilde{\mathcal{L}}_i$,
\begin{align}
\|(J_{\boldsymbol{\tilde{\mathcal{L}}}}(\theta)\!-\!J_{\boldsymbol{\tilde{\mathcal{L}}}}(\theta'))\,p(\boldsymbol{\tilde{\mathcal{L}}}(\theta'))\|&=\!
\Big\|\sum_{i=1}^m p_i(\boldsymbol{\tilde{\mathcal{L}}}(\theta'))\big(\nabla \tilde{\mathcal{L}}_i(\theta)-\nabla \tilde{\mathcal{L}}_i(\theta')\big)\Big\| \notag\\
&\le \max_i \zeta_i \,\|\theta-\theta'\|\notag.
\end{align}
Combining the two bounds proves
\[
\|\nabla g_\mu(\theta)-\nabla g_\mu(\theta')\|
\le
\Big(\frac{\max_i \ell_i^2}{\mu}+\max_i \zeta_i\Big)\,\|\theta-\theta'\|,
\qquad \forall \theta,\theta'\in \mathbb{R}^d.
\]
Let $\bar{\ell}=\max_{i}\ell_i$ and $\bar{\zeta}=\max_{i}\zeta_i$. Then, $g_{\mu}$ is $\Big(\frac{\bar{\ell}^2}{\mu}+\bar{\zeta}\Big)$-smooth on $\mathbb{R}^d$.
\end{proof}

\subsection{Proofs in Section~\ref{sec:app}}
\begin{proof}[Proof of Proposition~\ref{prop:smooth_approx}]
The first claim follows directly from Proposition~\ref{prop:bounded_approx} by taking the limit $\tau \rightarrow 0$, and we omit the proof for brevity. For the second claim, we define, for $\theta, \theta'\in \mathbb{R}^d$,
\[
\phi_i(\theta, \theta'):=\tilde{\mathcal{L}}_i(\theta')-\tilde{\mathcal{L}}_i(\theta),\qquad i\in[m].
\]
Since $\tilde{\mathcal{L}}_i$ is $\ell_i$-Lipschitz and $\zeta_i$-smooth on $\mathbb{R}^d$, $\phi_i$ is $\ell_i$-Lipschitz and $\zeta_i$-smooth w.r.t $\theta, \theta'\in \mathbb{R}^d$ as well.
Consider the log-sum-exp term
\[
h_{0,\tau}(\theta,\theta')=\tau\ln\!\Big(\sum_{i=1}^{m}\exp(\phi_i(\theta)/\tau)\Big).
\]
Applying Lemma~\ref{lem:lse_smoothness} with $\mu=\tau$ and $\{\phi_i\}$ gives that $h_{0,\tau}(\theta,\theta')$ is $\big(\frac{\bar{\ell}^2}{\tau}+\bar{\zeta}\big)$-smooth w.r.t $\theta, \theta'\in \mathbb{R}^d$. Since  $v_{0, \tau}(\theta) =-\min_{\theta' \in \mathbb{R}^d} h_{0,\tau}(\theta,\theta')$, we can also conclude that  $v_{0, \tau}(\theta)$ is also $\big(\frac{\bar{\ell}^2}{\tau}+\bar{\zeta}\big)$-smooth w.r.t $\theta$. Let $C=\bar{\zeta}$ and $\alpha = {\ell}^2$, we arrive the conclusion.
\end{proof}

\begin{proof}[Proof of Proposition~\ref{prop:bounded_approx}]
By the property of the log-sum-exp function (e.g.,~\cite{beck2012smoothing}), and since taking $\min$ preserves inequalities, we obtain
\begin{align}
\label{eq:prop1_bound}
u_\lambda(\theta) - \tau \ln m
\;\le\;
v_{\lambda,\tau}(\theta)
\;\le\;
u_\lambda(\theta),
\qquad \forall\,\theta\in \mathbb{R}^d .
\end{align}
When $\lambda=0,$ $u_0(\theta)= \bar{u}(\theta)$, we have 
$$v_{0,\tau}(\theta) \le \bar{u}(\theta)\leq v_{0,\tau}(\theta) + \tau\ln m. $$
Considering Eq.~\eqref{eq:prop1_bound}, it also implies that as $\tau\rightarrow 0$, $v_{\lambda,\tau}(\theta)$ uniformly converges to $u_\lambda(\theta)$. Also recall from Lemma~\ref{lem:u_lambda_merit} that $\theta$ is weakly Pareto optimal if and only if $u_0(\theta)=0$. Therefore, $\theta$ is weakly Pareto optimal if and only if $\lim_{\tau\rightarrow 0} v_{0,\tau}(\theta)=0$. which completes the proof.
\end{proof}

\begin{proof}[Proof of Lemma~\ref{lemmsc}]
By $\mu_i$-weak convexity, $\frac{1}{\tau}\tilde{\mathcal{L}}_i(\theta')$ is $\frac{\mu_i}{\tau}$-weakly convex with respect to $\theta'$. Therefore, combining with Lemma~\ref{lem:lse_weak_convex}, the function
\[
\tau\ln\!\left(
\sum_{i=1}^m \exp\!\left(\frac{\tilde{\mathcal{L}}_i(\theta)-\tilde{\mathcal{L}}_i(\theta')}{\tau}\right)
\right)
\]
is $\bar{\mu}$-weakly convex with respect to $\theta'$, where $\bar{\mu}=\max_{i\in[m]}\mu_i$. Since $\lambda+\max_{i\in[m]}\mu_i=\lambda+\bar{\mu}>0$, the function $h_{\lambda,\tau}(\theta,\theta')$ is strictly convex with respect to $\theta'$. Hence the minimizer of $\min_{\theta'\in \mathbb{R}^d} h_{\lambda,\tau}(\theta,\theta')$ is unique.
\end{proof}

\begin{proof}[Proof of Proposition~\ref{theov}]
\noindent\textbf{(1)} For the first claim, by Eq.~\eqref{eq:prop1_bound} and $u_\lambda(\theta)\geq 0$, we have $v_{\lambda,\tau}(\theta) \ge -\tau \ln m$ for any $\theta \in \mathbb{R}^d$. 

\noindent\textbf{(2)} For the second claim, by Lemma~\ref{lem:u_lambda_merit}, if $\theta$ is weakly Pareto optimal, then $u_{\lambda}(\theta)=0$. Furthermore, by Proposition~\ref{prop:bounded_approx}, we have $u_{\lambda}(\theta)\ge v_{\lambda,\tau}(\theta)$, which implies $v_{\lambda,\tau}(\theta)\le 0$.



\noindent\textbf{(3) $\lambda=0$.} If $v_{0,\tau}(\theta)\le 0$, then by Proposition~\ref{prop:bounded_approx} we have
\begin{equation}
\label{eq:prop2_casea_1}
\bar{u}(\theta) \le v_{0,\tau}(\theta)+\tau\ln m \le \tau\ln m.
\end{equation}
By the definition of $\bar{u}$, for all $\hat{\theta}\in \mathbb{R}^d$, it holds that
\begin{equation}
\label{eq:prop2_casea_2}
\min_{i\in[m]} \bigl\{\tilde{\mathcal{L}}_i(\theta)-\tilde{\mathcal{L}}_i(\hat{\theta})\bigr\} \le \bar{u}(\theta) \le \tau\ln m.
\end{equation}
Equivalently, there does not exist any $\hat{\theta}\in \mathbb{R}^d$ such that
\[
\tilde{\mathcal{L}}_i(\hat{\theta}) < \tilde{\mathcal{L}}_i(\theta) - \tau\ln m, \qquad \forall i\in[m].
\]
Since $\psi(\cdot)$ is strictly increasing, this implies that there is no $\hat{\theta}\in \mathbb{R}^d$ satisfying
\[
\mathcal{L}_i(\hat{\theta})< \mathcal{L}_i(\theta)-\epsilon, \qquad \forall i\in[m], \quad\text{with }\epsilon=\tau\ln m,
\]
i.e., $\theta$ is $\epsilon$-weakly Pareto optimal.

\noindent\textbf{(4) $\lambda>0$.}
If $v_{\lambda,\tau}(\theta)\le -\tau\ln m$, then Proposition~\ref{prop:bounded_approx} yields
\begin{align}
0 \le u_{\lambda}(\theta) \le v_{\lambda,\tau}(\theta)+\tau\ln m \le 0\notag,
\end{align}
hence $u_{\lambda}(\theta)=0$. By the definition of $u_{\lambda}$, this implies that
\begin{equation}
\label{eq:prop2_caseb_2}
\min_{i\in[m]}
\bigl\{\tilde{\mathcal{L}}_i(\theta)-\tilde{\mathcal{L}}_i(\theta')\bigr\}
-\frac{\lambda}{2}\|\theta-\theta'\|^2
\le 0,
\qquad \forall\,\theta'\in \mathbb{R}^d.
\end{equation}
For any $\hat{\theta}\in \mathbb{R}^d$ and any $t\in(0,1)$. Let
\[
\theta_t := (1-t)\theta+t\hat{\theta} \in \mathbb{R}^d.
\]
Plugging $\theta'=\theta_t$ into Eq.~\eqref{eq:prop2_caseb_2}, we obtain
\begin{equation}
\label{eq:prop2_caseb_3}
\min_{i\in[m]} \bigl\{\tilde{\mathcal{L}}_i(\theta)-\tilde{\mathcal{L}}_i(\theta_t)\bigr\} -\frac{\lambda}{2}\|\theta-\theta_t\|^2 \le 0.
\end{equation}
Assume that for all $i\in[m]$, $\tilde{\mathcal{L}}_i$ is convex at $\theta$. Then for each $i\in[m]$,
\[
\tilde{\mathcal{L}}_i(\theta_t) \le t\,\tilde{\mathcal{L}}_i(\hat{\theta})+(1-t)\tilde{\mathcal{L}}_i(\theta),
\]
which implies
\[
\tilde{\mathcal{L}}_i(\theta)-\tilde{\mathcal{L}}_i(\theta_t) \ge t\bigl(\tilde{\mathcal{L}}_i(\theta)-\tilde{\mathcal{L}}_i(\hat{\theta})\bigr).
\]
Substituting into Eq.~\eqref{eq:prop2_caseb_3} and using $\|\theta-\theta_t\|^2=t^2\|\hat{\theta}-\theta\|^2$ yields
\begin{align}
\min_{i\in[m]}
\Bigl\{t\bigl(\tilde{\mathcal{L}}_i(\theta)-\tilde{\mathcal{L}}_i(\hat{\theta})\bigr)\Bigr\} -\frac{\lambda}{2}\,t^2\|\hat{\theta}-\theta\|^2 \le 0 \notag.
\end{align}
Dividing both sides by $t>0$ and letting $t\rightarrow 0$, we obtain for all $\hat{\theta}\in \mathbb{R}^d$,
\begin{align}
\min_{i\in[m]} \bigl\{\tilde{\mathcal{L}}_i(\theta)-\tilde{\mathcal{L}}_i(\hat{\theta})\bigr\} \le 0 \notag,
\end{align}
which means that there does not exist $\hat{\theta}\in \mathbb{R}^d$ such that $\tilde{\mathcal{L}}_i(\hat{\theta})<\tilde{\mathcal{L}}_i(\theta)$ for all $i\in[m]$. Therefore, $\theta$ is weakly Pareto optimal.
\end{proof}

\subsection{Proofs in Section~\ref{analysis}}
\begin{proof}[Proof of Lemma~\ref{smoothnessh}]
According to proof of Proposition~\ref{prop:smooth_approx}, we know $h_{0,\tau}(\theta,\theta')$ is $\big(\frac{\bar{\ell}^2}{\tau}+\bar{\zeta}\big)$-smooth w.r.t $\theta, \theta'\in \mathbb{R}^d$. Then, we have $h_{\lambda,\tau}(\theta,\theta')$ is $\big(\frac{\bar{\ell}^2}{\tau}+\bar{\zeta}+ \lambda \big) $-smooth, where the $\lambda$ term comes from the quadratic regularizer $\frac{\lambda}{2}\|\theta-\theta'\|^2$.
\end{proof}

\begin{lemma}[Lemma 4.3 in~\cite{lin2020gradient}]
For TTGDA, let $\delta^k=\bigl\|\theta^{\star}_{\lambda,\tau}(\theta^k)-\theta'^k\bigr\|^2$,
Then the following statements hold true.
\begin{equation}
\label{eq:delta}
\delta^k\le\Bigl(1-\frac{1}{2\kappa}+4\kappa^3\ell^2\eta_{\theta}^2\Bigr)\delta^{k-1}+4\kappa^3\eta_{\theta}^2\bigl\|\nabla v_{\lambda,\tau}(\theta^{k-1})\bigr\|^2,
\end{equation}
and 
\begin{equation}
\label{eq:vdescent}
v_{\lambda,\tau}(\theta^k)\le v_{\lambda,\tau}(\theta^{k-1})-\frac{7\eta_{\theta}}{16}\bigl\|\nabla v_{\lambda,\tau}(\theta^{k-1})\bigr\|^2+\frac{9\eta_{\theta}\ell^2}{16}\delta^{k-1}.
\end{equation}
\end{lemma}

\begin{proof}[Proof of Theorem~\ref{thm:gda_complexity}]
Let
\[
\delta^k
=
\bigl\|
\theta^{\star}_{\lambda,\tau}(\theta^k)
-
\theta'^k
\bigr\|^2,
\qquad
g^{k}
:=
\bigl\|
\nabla v_{\lambda,\tau}(\theta^{k})
\bigr\|^2.
\]
Define
\[
\gamma
:=
1-\frac{1}{2\kappa}+4\kappa^3\ell^2\eta_\theta^2.
\]
By Eq.~\eqref{eq:delta}, for all $k\ge 1$,
\begin{equation}
\label{eq:rec_delta_C1}
\delta^k \le \gamma\,\delta^{k-1}+4\kappa^3\eta_\theta^2\, g^{k-1}.
\end{equation}
Unrolling Eq.~\eqref{eq:rec_delta_C1} yields
\begin{equation}
\label{eq:delta_unroll_C1}
\delta^k
\le
\gamma^k \delta^0
+
4\kappa^3\eta_\theta^2
\sum_{j=0}^{k-1}\gamma^{k-1-j} g^{j}.
\end{equation}

Next, by Eq.~\eqref{eq:vdescent}, for all $k\ge 1$,
\begin{equation}
\label{eq:descent_C1}
v_{\lambda,\tau}(\theta^k)
\le
v_{\lambda,\tau}(\theta^{k-1})
-
\frac{7\eta_\theta}{16}\, g^{k-1}
+
\frac{9\eta_\theta\ell^2}{16}\,\delta^{k-1}.
\end{equation}
Summing Eq.~\eqref{eq:descent_C1} over $k=1,\dots,K$ gives
\begin{align}
\frac{7\eta_\theta}{16}\sum_{k=0}^{K-1} g^{k}
\le
v_{\lambda,\tau}(\theta^{0})-v_{\lambda,\tau}(\theta^{K})
+
\frac{9\eta_\theta\ell^2}{16}\sum_{k=0}^{K-1}\delta^{k}\notag.
\end{align}
Using $v_{\lambda,\tau}(\theta^{K})\ge \min_{\theta}v_{\lambda,\tau}(\theta)$ and
$\Delta_v:=v_{\lambda,\tau}(\theta^{0})-\min_{\theta}v_{\lambda,\tau}(\theta)$,
we obtain
\begin{equation}
\label{eq:sum_descent_C1b}
\frac{7\eta_\theta}{16}\sum_{k=0}^{K-1} g^{k}
\le
\Delta_v
+
\frac{9\eta_\theta\ell^2}{16}\sum_{k=0}^{K-1}\delta^{k}.
\end{equation}

We now bound $\sum_{k=0}^{K-1}\delta^{k}$ using Eq.~\eqref{eq:delta_unroll_C1}. Summing Eq.~\eqref{eq:delta_unroll_C1} over $k=1,\dots,K$ and exchanging the order of summation yield
\begin{align}
\sum_{k=1}^{K}\delta^{k}
&\le
\delta^{0}\sum_{k=1}^{K}\gamma^{k}
+
4\kappa^3\eta_\theta^2
\sum_{k=1}^{K}\sum_{j=0}^{k-1}\gamma^{k-1-j} g^{j} \nonumber\\
&=
\delta^{0}\sum_{k=1}^{K}\gamma^{k}
+
4\kappa^3\eta_\theta^2
\sum_{j=0}^{K-1} g^{j}\sum_{k=j+1}^{K}\gamma^{k-1-j} \nonumber\\
&\le
\delta^{0}\frac{\gamma}{1-\gamma}
+
4\kappa^3\eta_\theta^2
\sum_{j=0}^{K-1} g^{j}\frac{1}{1-\gamma},
\label{eq:sum_delta_C1}
\end{align}
where we used $\sum_{k=1}^{K}\gamma^{k}\le \frac{\gamma}{1-\gamma}$ and $\sum_{k=j+1}^{K}\gamma^{k-1-j}\le \frac{1}{1-\gamma}$.

Substituting Eq.~\eqref{eq:sum_delta_C1} into Eq.~\eqref{eq:sum_descent_C1b} gives
\begin{align}
\frac{7\eta_\theta}{16}\sum_{k=0}^{K-1} g^{k}
&\le
\Delta_v
+
\frac{9\eta_\theta\ell^2}{16}
\left(
\delta^{0}\frac{\gamma}{1-\gamma}
+
4\kappa^3\eta_\theta^2\frac{1}{1-\gamma}
\sum_{k=0}^{K-1} g^{k}
\right).
\label{eq:pre_rearrange_C1}
\end{align}
Rearranging Eq.~\eqref{eq:pre_rearrange_C1} yields
\begin{equation}
\label{eq:rearranged_C1}
\left(
\frac{7\eta_\theta}{16}
-
\frac{9\eta_\theta\ell^2}{16}
\cdot
\frac{4\kappa^3\eta_\theta^2}{1-\gamma}
\right)
\sum_{k=0}^{K-1} g^{k}
\le
\Delta_v
+
\frac{9\eta_\theta\ell^2}{16}
\cdot
\frac{\gamma}{1-\gamma}\,\delta^{0}.
\end{equation}

Finally, choose $\eta_\theta=\Theta(1/(\kappa^2\ell))$ so that
$4\kappa^3\ell^2\eta_\theta^2 \le \frac{7}{32\kappa}$, which implies
\[
1-\gamma
=
\frac{1}{2\kappa}-4\kappa^3\ell^2\eta_\theta^2
\ge
\frac{9}{32\kappa},
\qquad
\frac{1}{1-\gamma}\le \frac{32}{9}\kappa,
\qquad
\frac{4\kappa^3\eta_\theta^2}{1-\gamma}\le \frac{128}{9}\kappa^4\eta_\theta^2.
\]
Under this choice, the coefficient on the left-hand side of Eq.~\eqref{eq:rearranged_C1} is lower bounded by a constant multiple of $\eta_\theta$, and thus
\begin{align}
\sum_{k=0}^{K-1} g^{k}
\le
\mathcal{O}\!\left(
\frac{\Delta_v}{\eta_\theta}
+
\kappa\ell^2\,\delta^{0}
\right)\notag.
\end{align}
Let $\widehat{k}$ be sampled uniformly from $\{0,\dots,K-1\}$. Then
\[
\mathbb{E}\bigl\|\nabla v_{\lambda,\tau}(\theta^{\widehat{k}})\bigr\|^2
=
\frac{1}{K}\sum_{k=0}^{K-1} g^{k}
\le
\mathcal{O}\!\left(
\frac{\Delta_v}{K\eta_\theta}
+
\frac{\kappa\ell^2\,\delta^{0}}{K}
\right).
\]
Setting the right-hand side to be at most $\epsilon^2$ and using
$\eta_\theta=\Theta(1/(\kappa^2\ell))$ gives the iteration (and gradient)
complexity
\[
K
=
\mathcal{O}\!\left(
\frac{\kappa^2\ell\,\Delta_v+\kappa\ell^2 \delta^{0}}{\epsilon^2}
\right).
\]
This concludes the proof.
\end{proof}

\begin{table*}[t]
\caption{Ablation studies on the impact of $\tau$ on the NYUv2 dataset. The best results for each task are shown in \textbf{bold}. $\uparrow$($\downarrow$) means that the higher (lower) the value, the better the performance.}
\label{table:ablation_tau_detail}
\centering
\setlength{\tabcolsep}{5pt}
\renewcommand{\arraystretch}{1.15}
\begin{tabular}{l cc cc cc ccc c}
\toprule
\multirow{3}{*}[-1ex]{$\tau$}
& \multicolumn{2}{c}{\textbf{Segmentation}}
& \multicolumn{2}{c}{\textbf{Depth}}
& \multicolumn{5}{c}{\textbf{Surface Normal}}
& \multirow{3}{*}[-1ex]{$\Delta_b \uparrow$} \\
\cmidrule(lr){2-3}\cmidrule(lr){4-5}\cmidrule(lr){6-10}
& \multirow{2}{*}[-0.4ex]{\textbf{mIoU}$\uparrow$}
& \multirow{2}{*}[-0.4ex]{\textbf{Pix Acc}$\uparrow$}
& \multirow{2}{*}[-0.4ex]{\textbf{Abs Err}$\downarrow$}
& \multirow{2}{*}[-0.4ex]{\textbf{Rel Err}$\downarrow$}
& \multicolumn{2}{c}{\textbf{Angle Distance}}
& \multicolumn{3}{c}{\textbf{Within $t^\circ$}}
& \\
\cmidrule(lr){6-7}\cmidrule(lr){8-10}
& & & & 
& \textbf{Mean}$\downarrow$
& \textbf{Median}$\downarrow$
& \textbf{11.25}$\uparrow$
& \textbf{22.5}$\uparrow$
& \textbf{30}$\uparrow$
& \\
\midrule
$\tau=0.1$
& 65.21 & 83.46
& 0.2982 & 0.1141
& 18.71 & 13.63
& 42.59 & 71.28 & 81.38
& +19.10\% \\
$\tau=0.5$
& 66.58 & 84.41
& 0.2822 & 0.1100
& 17.05 & 11.87
& 48.49 & 75.30 & 84.20
& +23.84\% \\
$\tau=1$
& \textbf{68.24} & \textbf{84.78}
& 0.2788 & 0.1094
& 16.93 & 11.81
& 48.71 & 75.62 & 84.49
& +24.80\% \\
$\tau=3$
& 66.92 & 84.43
& 0.2738 & 0.1068
& 16.79 & 11.69
& 49.14 & 76.07 & 84.76
& +25.06\% \\
$\tau=5$
& 67.12 & 84.70
& \textbf{0.2722} & \textbf{0.1050}
& 16.85 & 11.65
& 49.35 & 75.85 & 84.54
& +\textbf{25.44\%} \\
$\tau=10$
& 67.35 & 84.32
& 0.2772 & 0.1088
& \textbf{16.73} & \textbf{11.56}
& \textbf{49.59} & \textbf{76.15} & \textbf{84.77}
& +24.97\% \\
\bottomrule
\end{tabular}
\end{table*}

\begin{table*}[t]
\caption{The impact of $\lambda$ on the NYUv2 dataset. The best results for each task are shown in \textbf{bold}. $\uparrow$($\downarrow$) means that the higher (lower) the value, the better the performance.}
\label{table:ablation_lambda_detail}
\centering
\setlength{\tabcolsep}{5pt}
\renewcommand{\arraystretch}{1.15}
\begin{tabular}{c cc cc cc ccc c}
\toprule
\multirow{3}{*}[-1ex]{$\lambda$}
& \multicolumn{2}{c}{\textbf{Segmentation}}
& \multicolumn{2}{c}{\textbf{Depth}}
& \multicolumn{5}{c}{\textbf{Surface Normal}}
& \multirow{3}{*}[-1ex]{$\Delta_b \uparrow$} \\
\cmidrule(lr){2-3}\cmidrule(lr){4-5}\cmidrule(lr){6-10}
& \multirow{2}{*}[-0.4ex]{\textbf{mIoU}$\uparrow$}
& \multirow{2}{*}[-0.4ex]{\textbf{Pix Acc}$\uparrow$}
& \multirow{2}{*}[-0.4ex]{\textbf{Abs Err}$\downarrow$}
& \multirow{2}{*}[-0.4ex]{\textbf{Rel Err}$\downarrow$}
& \multicolumn{2}{c}{\textbf{Angle Distance}}
& \multicolumn{3}{c}{\textbf{Within $t^\circ$}}
& \\
\cmidrule(lr){6-7}\cmidrule(lr){8-10}
& & & & 
& \textbf{Mean}$\downarrow$
& \textbf{Median}$\downarrow$
& \textbf{11.25}$\uparrow$
& \textbf{22.5}$\uparrow$
& \textbf{30}$\uparrow$
& \\
\midrule
0
& 67.12 & 84.70
& \textbf{0.2722} & \textbf{0.1050}
& 16.85 & 11.65
& 49.35 & 75.85 & 84.54
& +25.44\% \\

$1.00\times10^{-6}$
& 67.57 & 84.67
& 0.2759 & 0.1080
& 16.84 & 11.69
& 49.14 & 75.98 & \textbf{84.70}
& +25.06\% \\

$1.00\times10^{-5}$
& \textbf{68.02} & 84.73
& 0.2742 & 0.1063
& \textbf{16.77} & \textbf{11.62}
& \textbf{49.42} & \textbf{76.00} & \textbf{84.70}
& \textbf{+25.57}\% \\

$1.00\times10^{-4}$
& 67.39 & \textbf{84.75}
& 0.2768 & 0.1080
& 16.83 & 11.66
& 49.26 & 75.86 & 84.57
& +25.00\% \\

$1.00\times10^{-3}$
& 67.21 & 84.64
& 0.2805 & 0.1091
& 17.00 & 11.77
& 48.85 & 75.51 & 84.32
& +24.41\% \\

\bottomrule
\end{tabular}
\end{table*}

\section{Experimental Details}

\subsection{Illustrative Example}
\label{appdix:example}
We detail the illustrative example in Section~\ref{sec:motivation}.
Let $\theta=(\theta_1,\theta_2)\in\mathbb{R}^2$, and consider the following intermediate objectives:
\begin{align*}
\tilde{\ell}_1(\theta) &= c_1(\theta) f_1(\theta) + c_2(\theta) g_1(\theta), \\
\tilde{\ell}_2(\theta) &= c_1(\theta) f_2(\theta) + c_2(\theta) g_2(\theta),
\end{align*}
where
\begin{align*}
f_1(\theta) &= \ln\!\Big(\max\big(\big|0.5(-\theta_1-7)-\tanh(-\theta_2)\big|,\;5\times 10^{-6}\big)\Big) + 6, \\
f_2(\theta) &= \ln\!\Big(\max\big(\big|0.5(-\theta_1+3)-\tanh(-\theta_2)+2\big|,\;5\times 10^{-6}\big)\Big) + 6, \\
g_1(\theta) &= \frac{(-\theta_1+7)^2 + 0.1\,(-\theta_2-8)^2}{10} - 20, \\
g_2(\theta) &= \frac{(-\theta_1-7)^2 + 0.1\,(-\theta_2-8)^2}{10} - 20, \\
c_1(\theta) &= \max\!\big(\tanh(0.5\theta_2),\,0\big), \qquad
c_2(\theta) = \max\!\big(\tanh(-0.5\theta_2),\,0\big).
\end{align*}

This example follows~\cite{liu2021conflict} and its modification in~\cite{navon2022multi}.
Building on the latter modification, we further adjust the objectives by adding a constant shift to keep both losses positive. Specifically, we define
\begin{equation*}
\hat{\ell}_1(\theta)=\tilde{\ell}_1(\theta)+30,\qquad \hat{\ell}_2(\theta)=\tilde{\ell}_2(\theta)+30,
\end{equation*}
and use the final objectives
\begin{equation*}
\ell_1(\theta)=0.1\,\hat{\ell}_1(\theta),\qquad \ell_2(\theta)=\hat{\ell}_2(\theta). 
\end{equation*}
\begin{equation*}
\text{and }\ell_1(\theta)=\,\hat{\ell}_1(\theta),\qquad \ell_2(\theta)=0.1\hat{\ell}_2(\theta). 
\end{equation*}

We use five initialization points $(-8.5,7.5)$, $(0.0,0.0)$, $(9.0,9.0)$, $(-7.5,-0.5)$, and $(9.0,-1.0)$. All methods use Adam with learning rate $10^{-3}$ for up to $35\mathrm{K}$ iterations. 
We approximate the Pareto front by evaluating an $800\times800$ uniform grid over $[-12,12]$ along each dimension and retaining the non-dominated points. At iteration $t$, we compute the distance from the current objective-space point $\mathbf{z}_t$ to the approximated Pareto set $\mathcal{P}$:
\[
d_t \;=\; \min_{\mathbf{p}\in\mathcal{P}} \|\mathbf{z}_t-\mathbf{p}\|_2 .
\]
If $d_t<0.05$, we stop early.

\begin{table*}[t]
\caption{The baseline with normalization on the NYUv2 dataset. $\uparrow$($\downarrow$) means that the higher (lower) the value, the better the performance.}
\label{table8}
\centering
\setlength{\tabcolsep}{5pt}
\renewcommand{\arraystretch}{1.15}
\begin{tabular}{c cc cc cc ccc c}
\toprule
\multirow{3}{*}[-1ex]{Method}
& \multicolumn{2}{c}{\textbf{Segmentation}}
& \multicolumn{2}{c}{\textbf{Depth}}
& \multicolumn{5}{c}{\textbf{Surface Normal}}
& \multirow{3}{*}[-1ex]{$\Delta_b \uparrow$} \\
\cmidrule(lr){2-3}\cmidrule(lr){4-5}\cmidrule(lr){6-10}
& \multirow{2}{*}[-0.4ex]{\textbf{mIoU}$\uparrow$}
& \multirow{2}{*}[-0.4ex]{\textbf{Pix Acc}$\uparrow$}
& \multirow{2}{*}[-0.4ex]{\textbf{Abs Err}$\downarrow$}
& \multirow{2}{*}[-0.4ex]{\textbf{Rel Err}$\downarrow$}
& \multicolumn{2}{c}{\textbf{Angle Distance}}
& \multicolumn{3}{c}{\textbf{Within $t^\circ$}}
& \\
\cmidrule(lr){6-7}\cmidrule(lr){8-10}
& & & & 
& \textbf{Mean}$\downarrow$
& \textbf{Median}$\downarrow$
& \textbf{11.25}$\uparrow$
& \textbf{22.5}$\uparrow$
& \textbf{30}$\uparrow$
& \\
\midrule
EW + EMA
& 67.19 & 84.32
& 0.2750 & 0.1076
& 17.04 & 11.93
& 48.24 & 75.35 & 84.27
& +24.52\% \\

STCH + Epoch-wise
& 66.83 & 84.24
& 0.2777 & 0.1095
& 17.23 & 12.07
& 47.76 & 74.92 & 83.94
& +23.79\% \\

SIMS (Ours)  & 67.12 & 84.70 & 0.2722 & 0.1050 & 16.85   & 11.65   & 49.35 & 75.85 & 84.54 & +25.44\% \\
\bottomrule
\end{tabular}
\end{table*}

\begin{table}[t]
\centering
\caption{Comparison of different methods on Ali-CCP.}
\label{tab:auc_results}
\begin{tabular}{lccc}
\toprule
Method & AUC/T1 & AUC/T2 & Avg AUC \\
\midrule
EW           & 0.5828 & 0.5774 & 0.5801 \\
FOOPS        & 0.5742 & 0.5779 & 0.5761 \\
STCH         & 0.5653 & 0.5979 & 0.5816 \\
SIMS (Ours)  & 0.5757 & 0.6012 & 0.5884 \\
\bottomrule
\end{tabular}
\end{table}

\begin{table}[t]
\centering
\caption{Effect of $\lambda$ on convergence and Pareto-front quality.}\label{table10}
\begin{tabular}{c c c}
\toprule
$\lambda$ & Mean Iter. to $r_k < 0.1$ & Mean Rank of dist to pf \\
\midrule
0      & 2463 & 1.750 \\
$1\times10^{-4}$ & 2526 & 2.333 \\
$1\times10^{-3}$ & 1385 & 2.917 \\
$1\times10^{-2}$ & 711  & 2.917 \\
$1\times10^{-1}$ & 198  & 2.833 \\
\bottomrule
\end{tabular}
\end{table}

\begin{table}[t]
\centering
\caption{Effect of different $\lambda$ values on convergence speed and distance to the Pareto front for problem F11.}
\label{tab:f11_lambda}
\resizebox{\columnwidth}{!}{
\begin{tabular}{cccccc}
\toprule
Problem & Init Point & $\lambda$ & Iter. to $r_k < 0.1$ & Dist. to PF & Rank \\
\midrule
F11 & $(-2.0, 2.0)$ & $0$       & $98$ & $0.000453952$ & $1$ \\
F11 & $(-2.0, 2.0)$ & $10^{-4}$ & $97$ & $0.000637059$ & $2$ \\
F11 & $(-2.0, 2.0)$ & $10^{-3}$ & $95$ & $0.002296451$ & $3$ \\
F11 & $(-2.0, 2.0)$ & $10^{-2}$ & $15$ & $0.00230228$  & $4$ \\
F11 & $(-2.0, 2.0)$ & $10^{-1}$ & $2$  & $0.004238246$ & $5$ \\
\bottomrule
\end{tabular}
}
\end{table}

\begin{table*}[t]
\caption{Performance on the NYUv2 dataset when the depth loss is rescaled by $100\times$. $\uparrow$($\downarrow$) means that the higher (lower) the value, the better the performance.}
\label{tab:rescale}
\centering
\setlength{\tabcolsep}{5pt}
\renewcommand{\arraystretch}{1.15}
\begin{tabular}{c cc cc cc ccc c}
\toprule
\multirow{3}{*}[-1ex]{Method}
& \multicolumn{2}{c}{\textbf{Segmentation}}
& \multicolumn{2}{c}{\textbf{Depth}}
& \multicolumn{5}{c}{\textbf{Surface Normal}}
& \multirow{3}{*}[-1ex]{$\Delta_b \uparrow$} \\
\cmidrule(lr){2-3}\cmidrule(lr){4-5}\cmidrule(lr){6-10}
& \multirow{2}{*}[-0.4ex]{\textbf{mIoU}$\uparrow$}
& \multirow{2}{*}[-0.4ex]{\textbf{Pix Acc}$\uparrow$}
& \multirow{2}{*}[-0.4ex]{\textbf{Abs Err}$\downarrow$}
& \multirow{2}{*}[-0.4ex]{\textbf{Rel Err}$\downarrow$}
& \multicolumn{2}{c}{\textbf{Angle Distance}}
& \multicolumn{3}{c}{\textbf{Within $t^\circ$}}
& \\
\cmidrule(lr){6-7}\cmidrule(lr){8-10}
& & & & 
& \textbf{Mean}$\downarrow$
& \textbf{Median}$\downarrow$
& \textbf{11.25}$\uparrow$
& \textbf{22.5}$\uparrow$
& \textbf{30}$\uparrow$
& \\
\midrule
EW
& 57.58 & 78.94
& 0.2755 & 0.1085
& 18.95 & 13.46
& 43.32 & 70.89 & 80.64
& +17.39\% \\

STCH
& 67.01 & 84.12
& 0.2835 & 0.1122
& 17.35 & 12.24
& 47.12 & 74.65 & 83.80
& +23.01\% \\

foops
& 65.23 & 83.26
& 6.0287 & 2.9620
& 19.03 & 14.02
& 41.19 & 70.40 & 80.87
& -536.99\% \\

SIMS
& 67.84 & 84.85
& 0.2778 & 0.1078
& 16.80 & 11.60
& 49.40 & 76.02 & 84.66
& +25.21\% \\
\bottomrule
\end{tabular}
\end{table*}

\subsection{Implementation Details}
\label{app:mtldetails}
We employ the following loss functions:

\noindent \textbf{Semantic $\&$ Human Parts Segmentation.}
Both tasks use pixel-wise Cross-Entropy Loss over valid pixels:
$$\mathcal{L}_{seg} = - \frac{1}{N_{valid}} \sum_{p \in \Omega} \mathbf{1}_{[y_p \neq \text{ignore}]} \log(P(y_p | x_p))$$
where $\Omega$ is the image domain, $y_p$ is the ground truth label, and the loss ignores the void class (index 255).

\noindent \textbf{Depth Estimation.}
We use the $L_1$ Loss (Mean Absolute Error) restricted to pixels with valid depth annotations:
$$\mathcal{L}_{depth} = \frac{1}{N_{valid}} \sum_{p \in \Omega} \mathbf{1}_{[d_p > 0]} | d_p - \hat{d}_p |$$where $d_p$ and $\hat{d}_p$ denote the ground truth and predicted depth, respectively.

\noindent \textbf{Surface Normal Prediction.}
We align predicted normals with ground truth via Inverse Cosine Similarity:
$$\mathcal{L}_{normal} = 1 - \frac{1}{N_{valid}} \sum_{p \in \Omega} \mathbf{1}_{[m_p = 1]} \langle \mathbf{n}_p, \hat{\mathbf{n}}_p \rangle$$where $\mathbf{n}_p, \hat{\mathbf{n}}_p$ are unit vectors, and $\langle \cdot, \cdot \rangle$ denotes the dot product.

\noindent \textbf{Edge Detection.}
To handle the extreme sparsity of edge pixels, we use a Class-Balanced Binary Cross-Entropy Loss:
$$\mathcal{L}_{edge} = - \sum_{p \in \Omega} \left[ \beta y_p \log(\hat{y}_p) + (1-y_p) \log(1-\hat{y}_p) \right]$$
where $\beta$ is a positive weight (set to 0.95) to increase the penalty for missing edge pixels.

\subsection{Metric for Each Task}
\label{appdix:metric}
\noindent \textbf{NYUv2 dataset.} For semantic segmentation, we use mean Intersection over Union (mIoU) and Pixel Accuracy (Pix Acc); for depth prediction, Absolute Error (Abs Err) and Relative Error (Rel Err); and for surface normal estimation, mean and median angular errors in degrees, pixel percentages within 11.25, 22.5, and 30 degrees.

\noindent \textbf{CityScapes dataset.} For semantic segmentation, using mean Intersection over Union (mIoU) and Pixel Accuracy (Pix Acc); for depth prediction, Absolute Error (Abs Err) and Relative Error (Rel Err).

\noindent \textbf{PASCAL-Context dataset.} For semantic segmentation, human parts segmentation, and saliency estimation, we use mean Intersection over Union (mIoU); for surface normal estimation, mean angular error in degrees.

\noindent \textbf{The selection of $\tau$.}
As indicated by our ablation study, the performance is relatively insensitive to the choice of $\tau$ within a reasonable range. Specifically, for the reported results, we adopt $\tau=5$ for NYUv2 and $\tau=3$ for both Cityscapes and PASCAL-Context.

\section{Additional Experimental Studies}

\subsection{Effect of $\tau$}

\label{appdix:tau}
The complete table is shown in the Table.~\ref{table:ablation_tau_detail}.

\subsection{Effect of $\lambda$}
\label{appdix:lambda}
The complete table is shown in the Table.~\ref{table:ablation_lambda_detail}.

\begin{add}

\section{Additional Experiment}
\subsection{Normalization baselines}
We add comparisons with normalization baselines. \zbedit{Table~\ref{table8} shows the results of the baseline (EW and STCH) with normalization (EMA and Epoch-wise) on the NYUv2 dataset.} We see that EW + EMA improves over vanilla EW, but it still underperforms SIMS (24.52\% vs. 25.44\%). In contrast, STCH + Epoch-wise shows no gain (gain=-0.02\%), suggesting that simply attaching online normalization to a merit-based objective does not necessarily yield a better trade-off and may even disturb the original learning objective. Overall, these results indicate that the gain of SIMS comes from its principled scale-invariant scalarization design with weak Pareto property.
\subsection{Non-vision MTL scenarios}
We add a non-vision experiment on Ali-CCP \cite{ma2018entire}, where the CVR loss is 1–2 orders of magnitude smaller than CTR. \zbedit{Table~\ref{tab:auc_results} shows the results.} On a standard Shared-Bottom model, SIMS achieves the best average AUC (0.5884), mainly by improving the smaller-scale task. This supports the generalization of SIMS beyond vision tasks.
\subsection{Synthetic Experiments on $\lambda$}
To better address the concern on the gap between the theory and the empirical setting of $\lambda$, we added additional synthetic experiments. According to the bi-objective BBOB \cite{brockhoff2022using} benchmark, we selected F1, F2, and F11 because they provide a controlled progression in difficulty and geometric structure.
Specifically, F1 is formed by two Sphere functions, F2 by a Sphere function and a separable Ellipsoid function, and F11 by two separable Ellipsoid functions. We set centers as $(-3.0,-3.0)$ and $(3.0,3.0)$ for F1, $(-3.0,-2.0)$ and $(3.0,2.0)$ for F2, and $(-3.0,-2.0)$ and $(3.0,2.0)$ for F11, and all constant shift to be 1.

We choice $\lambda$ from the set $\{0, 10^{-1}, 10^{-2}, 10^{-3}, 10^{-4}\}$ and the initial points from the set $\{(-2.0,2.0), (2.0,-2.0), (3.0,1.0),(5.0,-5.0)\}$ to examine both convergence behavior and solution quality.
Empirically, we find that all tested values of $\lambda$ converge on these problems. Moreover, under the overall ranking across the three benchmarks and multiple initializations, where the ranking is defined by ordering different $\lambda$ values according to the distance between their final converged solutions and the Pareto front (pf), with smaller distance corresponding to a lower rank value, $\lambda=0$ achieves the smallest distance to the Pareto front in most cases (Mean Rank=1.75, \zbedit{as shown in Table~\ref{table10}}). This is consistent with our analysis: while a larger positive $\lambda$ improves the conditioning of the inner problem, it also changes the surrogate itself and may introduce a bias relative to the original merit-function characterization, whereas $\lambda=0$ remains the closest smooth approximation to the original objective.

To better characterize the role of $\lambda$, we further introduce the stationarity residual $$ r_k=\sqrt{|\nabla_{\theta} h_{\lambda,\tau}(\theta_k,\theta_k')|^2+|\nabla_{\theta'} h_{\lambda,\tau}(\theta_k,\theta_k')|^2 }.$$ Here $r_k$ can be used to measure how close the current solution to a first-order stationary point. We use $r_k$ to study the convergence to an $\epsilon$-stationary point in the $\lambda> \bar\mu$ regime. \zbedit{From the results in Table~\ref{table10},} we observe that as $\lambda$ increases, the mean first hitting iteration for $r_k<0.1$ becomes smaller, indicating that larger $\lambda$ generally leads to faster stabilization of the optimization dynamics. This is also in line with our theory, since a larger positive $\lambda$ improves the curvature of the inner problem and makes the optimization procedure better conditioned.

In summary, these observations indicate that choosing $\lambda$ is essentially a trade-off (Table~\ref{tab:f11_lambda} provides a representative example on problem F11 with initialization at $(-2.0,, 2.0)$). A larger positive $\lambda$ improves the conditioning of the inner problem and tends to accelerate the convergence, while a smaller $\lambda$ preserves a more faithful characterization of the original merit-function objective and often yields a smaller distance to the Pareto front (Sec.\ref{sec:ablation}). Hence, $\lambda$ balances the convergence rate and fidelity to the original problem formulation. How to choose $\lambda$ in a more principled or adaptive way is an interesting direction that deserves further investigation in our future work.

\subsection{Rescale task losses on a real benchmark}
We additionally rescale the depth loss by $100\times$ on NYUv2 to evaluate the robustness of different methods under real-data task-wise rescaling. As shown in Table \ref{tab:rescale} and \ref{table:nyu}, EW is clearly affected by the enlarged depth loss, with $\Delta_b$ dropping from $22.34\%$ to $17.39\%$. Although STCH adopts objective normalization, it still exhibits a mild performance decline, with $\Delta_b$ decreasing from $23.81\%$ to $23.01\%$. FOOPS is even more severely affected by the rescaling, the optimization on the depth task collapses, leading to extremely large depth errors and also degrading the performance of the other tasks, which results in a substantially worse overall score. In contrast, SIMS remains highly stable, with $\Delta_b$ only changing from $25.44\%$ to $25.21\%$. These results provide direct real-data evidence that artificial task-wise rescaling can significantly affect the scalarization methods, while SIMS remains robust under loss rescaling, consistent with our scale-invariance claim.

\end{add}

\end{document}